%% file: ms.tex
\documentclass[runningheads]{llncs}

\makeatletter
\let\@internalcite\cite
\def\cite{\def\citeauthoryear##1##2{##1, ##2}\@internalcite}
\def\shortcite{\def\citeauthoryear##1##2{##2}\@internalcite}
\def\@biblabel#1{\def\citeauthoryear##1##2{##1, ##2}[#1]\hfill}
\makeatother

\usepackage{multirow}

\usepackage{subcaption}
\usepackage{amssymb}
\usepackage{amsmath}
\usepackage{amsthm}
\usepackage{graphicx}

\newcommand{\spectralfpl}{{\tt SpectralLeader}}

\usepackage[ruled, vlined, linesnumbered]{algorithm2e}

\makeatletter
\newcommand{\nosemic}{\renewcommand{\@endalgocfline}{\relax}}
\newcommand{\dosemic}{\renewcommand{\@endalgocfline}{\algocf@endline}}
\let\oldnl\nl
\newcommand{\nonl}{\renewcommand{\nl}{\let\nl\oldnl}}

\DeclareMathOperator*{\argmin}{argmin}

\SetCommentSty{mycommfont}

\usepackage{comment}

\usepackage[usenames,dvipsnames,svgnames,table]{xcolor}

\usepackage[capitalize,noabbrev]{cleveref}

\usepackage[backgroundcolor=White]{todonotes}

\newcommand{\commentout}[1]{}
\newcommand{\junk}[1]{}
\newcommand{\etal}{\emph{et al.}}

\SetKwInput{KwData}{Input}
\SetKwInput{KwResult}{Output}

\begin{document}
\title{$\spectralfpl$: Online Spectral Learning for Single Topic Models}
%
%
\author{Tong Yu\inst{1} \and
Branislav Kveton \inst{2}\and
Zheng Wen\inst{2} \and
Hung Bui\inst{3} \and
Ole J. Mengshoel\inst{1}}
%
%
\institute{Carnegie Mellon University \and
Adobe Research \and Google DeepMind\\
\email{tongy1@andrew.cmu.edu},
\email{\{kveton,zwen\}@adobe.com},
\email{bui.h.hung@gmail.com},
\email{ole.mengshoel@sv.cmu.edu}}
\maketitle   

\begin{abstract}
We study the problem of learning a latent variable model from a stream of data. Latent variable models are popular in practice because they can explain observed data in terms of unobserved concepts. These models have been traditionally studied in the offline setting. In the online setting, on the other hand, the online EM is arguably the most popular algorithm for learning latent variable models. Although the online EM is computationally efficient, it typically converges to a local optimum. In this work, we develop a new online learning algorithm for latent variable models, which we call $\spectralfpl$. $\spectralfpl$ always converges to the global optimum, and we derive a sublinear upper bound on its $n$-step regret in the bag-of-words model. In both synthetic and real-world experiments, we show that $\spectralfpl$ performs similarly to or better than the online EM with tuned hyper-parameters.
  \end{abstract}

\input{Introduction}
\input{RelatedWork}
\input{SpectralMethodforTopicModel}
\input{Setting}
\input{Algorithm}

\input{Analysis}
\input{Experiments}
\input{Conclusion}

\bibliographystyle{splncs04}
\bibliography{references}



\end{document}

%% file: Introduction.tex

\section{Introduction}
\label{sec:intro}

Latent variable models are classical approaches to explain observed data through unobserved concepts. They have been successfully applied in a wide variety of fields, such as speech recognition, natural language processing, and computer vision \cite{rabiner1989tutorial,wallach2006topic,nowozin2011structured,bishop2006pattern}. Despite their successes, latent variable models are typically studied in the offline setting. However, in many practical problems, a learning agent needs to learn a latent variable model online while interacting with real-time data with unobserved concepts. For instance, a recommender system may want to learn to cluster its users online based on their real-time behavior. This paper aims to develop algorithms for such online learning problems.

Previous works proposed several algorithms to learn latent variable models online by extending the expectation maximization (EM) algorithm. Those algorithms are known as online EM algorithms, and include the stepwise EM \cite{cappe2009line,liang2009online} and the incremental EM \cite{neal1998view}. Similar to the offline EM, each iteration of an online EM algorithm includes an E-step to fill in the values of latent variables based on their estimated distribution, and an M-step to update the model parameters. The main difference is that each step of online EM algorithms only uses data received in the current iteration, rather than the whole dataset. This ensures that online EM algorithms are computationally efficient and can be used to learn latent variable models online. However, similarly to the offline EM, online EM algorithms have one major drawback: they may converge to a local optimum and hence suffer from a non-diminishing performance loss.

To overcome these limitations, we develop an online learning algorithm that performs almost as well as the globally optimal latent variable model, which we call $\spectralfpl$. Specifically, we propose an online learning variant of the spectral method \cite{anandkumar2014tensor}, which can learn the parameters of latent variable models offline with guarantees of convergence to a global optimum. Our online learning setting is defined as follows. We have a sequence of $n$ topic models, one at each time $t \in [n]$. The prior distribution of topics can change arbitrarily over time, while the conditional distribution of words is stationary. At time $t$, the learning agent observes a document of words, which is sampled i.i.d. from the model at time $t$. The goal of the agent is to predict a sequence of model parameters with low cumulative regret with respect to the best solution in hindsight, which is constructed based on the sampling distribution of the words over $n$ steps.

This paper makes several contributions. First, it is the first paper to formulate online learning with the spectral method as a regret minimization problem. Second, we propose $\spectralfpl$, an online learning variant of the spectral method for our problem. To reduce computational and space complexities of $\spectralfpl$, we introduce reservoir sampling. Third, we prove a sublinear upper bound on the $n$-step regret of $\spectralfpl$. Finally, we compare $\spectralfpl$ to the stepwise EM in extensive synthetic and real-world experiments. We observe that the stepwise EM is sensitive to the setting of its hyper-parameters. In all experiments, $\spectralfpl$ performs similarly to or better than the stepwise EM with tuned hyper-parameters.

%% file: RelatedWork.tex

\section{Related Work}
The spectral method by tensor decomposition has been widely applied in different latent variable models, such as mixtures of tree graphical models \cite{anandkumar2014tensor}, mixtures of linear regressions \cite{chaganty2013spectral}, hidden Markov models (HMM) \cite{anandkumar2012method}, latent Dirichlet allocation (LDA) \cite{anandkumar2012spectral}, Indian buffet process \cite{tung2014spectral}, and hierarchical Dirichlet process \cite{tung2017spectral}.  One major advantage of the spectral method is that it learns globally optimal solutions \cite{anandkumar2014tensor}. The spectral method first empirically estimates low-order moments of observations and then applies decomposition methods with a unique solution to recover the model parameters. 

Traditional online learning methods for latent variable models usually extend traditional iterative methods for learning latent variable models in the offline setting. Offline EM calculates the sufficient statistics based on all the data, while in online EM the sufficient statistics are updated with the recent data in each iteration \cite{cappe2009line,neal1998view,liang2009online}. Online variational inference is used to learn LDA efficiently \cite{hoffman2010online}. These online algorithms converge to local minima, while we aim to develop an algorithm with a theoretical guarantee of convergence to a global optimum. 

An online spectral learning method has also been developed \cite{huang2015online}, with a focus on improving computational efficiency, by conducting optimization of multilinear operations in SGD and avoiding directly forming the tensors. Online stochastic gradient for tensor decomposition has been analyzed \cite{ge2015escaping} with a different online setting: they do not look at the online problem as regret minimization and the analysis focuses on convergence to a local minimum. In contrast, we develop an online spectral method with a theoretical guarantee of convergence to a global optimum. Besides, our method is robust in the non-stochastic setting where the topics of documents are correlated over time. This non-stochastic setting has not been previously studied in the context of online spectral learning \cite{huang2015online}.

%% file: SpectralMethodforTopicModel.tex

\section{Spectral Method for Topic Model}
\label{sec:spectraltopic}

This section introduces the spectral method in latent variable models. Specifically, we describe how the method works in the simple bag-of-words model \cite{anandkumar2014tensor}. 

In the bag-of-words model, the goal is to understand the latent topic of documents based on the observed words in each document. 
Without loss of generality, we describe the spectral method and $\spectralfpl$ (\cref{sec:spectralleader}) in the setting where the number of words in each document is $3$. The extension to more than $3$ words is straightforward \cite{anandkumar2014tensor}. Let the number of distinct topics be $K$ and the size of the vocabulary be $d$. Then our model can be viewed as a mixture model, where the observed words $\mathbf{x}^{(1)}$, $\mathbf{x}^{(2)}$, and $\mathbf{x}^{(3)}$ are conditionally i.i.d. given topic $C$, which is also i.i.d.. Later in \cref{sec:setting}, we study a more general setting where the topics are non-stationary, in the sense that the distributions of topics can change arbitrarily over time. Each word is one-hot encoded, $\mathbf{x}^{(l)} = e_i$ if and only if $\mathbf{x}^{(l)}$ represents word $i$, where $e_1, \ldots, e_d$ is the standard coordinate basis in $\mathbb{R}^d$. The model is parameterized by the probability of each topic $j$, $\omega_j = P(C =j)$ for $j \in [K]$, and the conditional probability of all words $u_j \in [0, 1]^d$ given topic $j$. The $i$th entry of $u_j$ is $u_{j}(i) = P(\mathbf{x}^{(l)} = e_i|C=j)$ for $i \in [d]$. With $3$ observed words, it suffices to construct a third order tensor $\bar{M}_3$ as
\begin{align*}
  \mathbb{E}[\mathbf{x}^{(1)} \otimes \mathbf{x}^{(2)} \otimes \mathbf{x}^{(3)}] =
  \sum_{1 \leq i, j, k \leq d} P(\mathbf{x}^{(1)} = e_i, \mathbf{x}^{(2)} = e_j, \mathbf{x}^{(3)} = e_k) \ e_i \otimes e_j \otimes e_k.
\end{align*}
To recover the parameters of the topic model, we want to decompose $\bar{M}_3$ as 
\begin{align}
\label{eq:decomposeM}
\bar{M}_3 = \sum_{i=1}^K \omega_i u_i \otimes u_i \otimes u_i.
\end{align}
Unfortunately, such a decomposition is generally NP-hard \cite{anandkumar2014tensor}. Instead, we can decompose an orthogonal decomposable tensor. One way to make $\bar{M}_3$ orthogonal decomposable is by whitening. We can define a whitening matrix as $\bar{W} = U A^{-1/2}$, where $A \in \mathbb{R}^{K\times K}$ is the diagonal matrix of positive eigenvalues of $\bar{M}_2 = \mathbb{E}[\mathbf{x}^{(1)} \otimes \mathbf{x}^{(2)}] = \sum_{i=1}^K \omega_i u_i \otimes u_i$, and $U \in \mathbb{R}^{d\times K}$ is the matrix of $K$ eigenvectors associated with those eigenvalues. After whitening, instead of decomposing $\bar{M}_3$, we can decompose $\bar{T} = \mathbb{E}[\bar{W}^\top\mathbf{x}^{(1)} \otimes \bar{W}^\top\mathbf{x}^{(2)} \otimes \bar{W}^\top\mathbf{x}^{(3)}]$ as $\bar{T} = \sum_{i=1}^K \lambda_i v_i \otimes v_i \otimes v_i$ by the \emph{power iteration method} \cite{anandkumar2014tensor}. Finally, the model parameters are recovered by $\omega_i = \frac{1}{\lambda_i^2}$ and $u_i = \lambda_i  (\bar{W}^\top)^+ v_i$, where $(\bar{W}^\top)^+$ is the pseudoinverse of $\bar{W}^\top$. In practice, only a noisy realization of $\bar{T}$, $T$, is typically available, and it is constructed from empirical counts. Such tensors can be decomposed approximately, and the error of such a decomposition is analyzed in Theorem 5.1 of Anandkumar \etal~\shortcite{anandkumar2014tensor}.

%% file: Setting.tex

\section{Online Learning for Topic Models}
\label{sec:setting}

We study the following online learning problem in a single topic model. We have a sequence of $n$ topic models, one at each time $t \in [n]$. The prior distribution of topics can change arbitrarily over time, while the conditional distribution of words is stationary. We denote by $\mathbf{x}_t = (\mathbf{x}_t^{(l)})_{l = 1}^3$ a tuple of one-hot encoded words in the document at time $t$, which is sampled i.i.d. from the model at time $t$. Non-stationary distributions of topics are common in practice. For instance, in the recommender system example in \cref{sec:intro}, user clusters tend to be correlated over time. The clusters can be viewed as topics.

We represent the distribution of words at time $t$ by a cube $P_t = \mathbb{E}[\mathbf{x}_t^{(1)}\otimes\mathbf{x}_t^{(2)}\otimes\mathbf{x}_t^{(3)}]  \in [0, 1]^{d \times d \times d}$. In particular, the probability of observing the triplet of words $(i, j, k)$ at time $t$ is
\begin{align}
  P_t(i,j,k) = \sum_{c=1}^{K}P_t(c)P(\mathbf{x}_t^{(1)} = e_i|c)P(\mathbf{x}_t^{(2)}= e_j|c)P(\mathbf{x}_t^{(3)}= e_k|c)\,,
\end{align}
where $P_t(c)$ is the prior distribution of topics at time $t$. This prior distribution can change arbitrarily with $t$.

The learning agent predicts the distribution of words $\hat{M}_{3, t - 1} \in [0, 1]^{d \times d \times d}$ at time $t$ and is evaluated by its per-step loss $\ell_t(\hat{M}_{3, t - 1})$. The agent aims to minimize its cumulative loss, which measures the difference between the predicted distribution $\hat{M}_{3, t - 1}$ and the observations $\mathbf{x}_t^{(1)}\otimes\mathbf{x}_t^{(2)}\otimes\mathbf{x}_t^{(3)}$ over time.

But what should the loss be? In this work, we define the \emph{loss} at time $t$ as
\begin{align}
\label{eq:ourloss}
  \ell_t(M) = \| \mathbf{x}_t^{(1)}\otimes\mathbf{x}_t^{(2)}\otimes\mathbf{x}_t^{(3)} -M \|_F^{2}, 
\end{align}
where $\|.\|_F$ is the \emph{Frobenius norm}. For any tensor $M \in \mathbb{R}^{d \times d \times d}$, we define its Frobenius norm as $\|M\|_F = \sqrt{\sum_{i,j,k=1}^d M(i, j, k)^2}$. This choice can be justified as follows. Let
\begin{align}
\bar{M}_{3, n} = \frac{1}{n}\sum_{t=1}^{n}P_t = \frac{1}{n}\sum_{t=1}^{n}\mathbb{E}[\mathbf{x}_t^{(1)}\otimes\mathbf{x}_t^{(2)}\otimes\mathbf{x}_t^{(3)}]
\label{eq:barmn}
\end{align}
be the average of distributions from which $\mathbf{x}_t^{(1)}\otimes\mathbf{x}_t^{(2)}\otimes\mathbf{x}_t^{(3)}$ are generated in $n$ steps. Then
\begin{align}
\label{eq:motivation_loss}
\bar{M}_{3, n} = \argmin_{M \in [0, 1]^{d \times d \times d}} \sum_{t = 1}^n \mathbb{E}[\ell_t(M)]\,,
\end{align}
as shown in \cref{lem:argmin} in \cref{sec:technical lemmas}. In other words, the loss function is chosen such that a natural \emph{best solution in hindsight}, $\bar{M}_{3, n}$ in \eqref{eq:motivation_loss}, is the minimizer of the cumulative loss.

With the definition of the loss function and the best solution in hindsight, the goal of the learning agent is to minimize the regret
\begin{align}
  R(n) =
  \sum_{t = 1}^n \mathbb{E}[\ell_t(\hat{M}_{3, t - 1}) - \ell_t(\bar{M}_{3, n})]\,,
  \label{eq:regret}
\end{align}
where $\ell_t(\hat{M}_{3, t - 1})$ is the loss of our estimated model at time $t$ and $\ell_t(\bar{M}_{3, n})$ is the loss of the best solution in hindsight, respectively.

Unlike traditional online algorithms that minimize the negative log-likelihood \cite{liang2009online}, we minimize the parameter recovery loss. In the offline setting, the spectral method minimizes the recovery loss in a wide range of models \cite{anandkumar2014tensor,chaganty2013spectral,shaban2015learning}.

%% file: Algorithm.tex

\section{Algorithm $\spectralfpl$}
\label{sec:spectralleader}

We propose $\spectralfpl$, an online learning algorithm for minimizing the regret in \eqref{eq:regret}. Its pseudocode is in \cref{alg:spectralmethodnoisy}. At each time $t$, the input is the observation $(\mathbf{x}_t^{(l)})_{l = 1}^3$. We also maintain a set of reservoir samples $((\mathbf{x}_z^{(l)})_{l = 1}^3)_{z \in \mathcal{S}_{t-1}}$, where $\mathcal{S}_{t-1}$ is the set of the time indices of these reservoir samples from the previous $t-1$ time steps.

{\Huge
\begin{algorithm}[t]
\normalsize

\caption{$\spectralfpl$ at time $t$}
 \label{alg:spectralmethodnoisy}
 \KwData{$(\mathbf{x}_t^{(l)})_{l = 1}^3$} 

$M_{2,t-1} \leftarrow \frac{1}{|\mathcal{S}_{t-1}| |\Pi_2(3)|} \sum_{z \in \mathcal{S}_{t-1}}\sum_{\pi \in \Pi_2(3)} \mathbf{x}_z^{(\pi(1))} \otimes \mathbf{x}_z^{(\pi(2))}$

$W_{t-1} \leftarrow U_{t-1}A_{t-1}^{-1/2}$ 

\dosemic\nonl//~$A_{t-1} \in \mathbb{R}^{K\times K}$  is the diagonal matrix of $K$ positive eigenvalues of $M_{2,t-1}$, and $U_{t-1} \in \mathbb{R}^{d\times K}$  is the matrix of eigenvectors associated with these positive eigenvalues.


$T_{t-1} \leftarrow \frac{1}{|\mathcal{S}_{t-1}| |\Pi_3(3)|} \sum_{z \in \mathcal{S}_{t-1}} \sum_{\pi \in \Pi_3(3)} W_{t-1}^\top \mathbf{x}_z^{(\pi(1))} \otimes W_{t-1}^\top \mathbf{x}_z^{(\pi(2))} \otimes W_{t-1}^\top \mathbf{x}_z^{(\pi(3))}$

Obtain $(\lambda_{t - 1,i})_{i = 1}^K$ and $(v_{t - 1,i})_{i = 1}^K$ from $T_{t-1}$ by power iteration method

$\omega_{t-1,i} \leftarrow \frac{1}{\lambda_{t-1,i}^2}, \quad u_{t-1,i} \leftarrow \lambda_{t-1,i}  (W_{t-1}^\top)^+ v_{t-1,i}$ for all $i \in [K]$


Generate a random number $a \in [0, 1]$

\If {$t \leq m$}{$\mathcal{S}_{t} \leftarrow \mathcal{S}_{t-1}\cup \{t\}$}
\ElseIf {$a \leq m / (t-1)$}{Remove a random sample from $\mathcal{S}_{t-1}$ \\ $\mathcal{S}_{t} \leftarrow \mathcal{S}_{t-1} \cup \{t\}$} 
\Else{$\mathcal{S}_{t} \leftarrow \mathcal{S}_{t-1}$}
	
\KwResult{Model parameters $\omega_{t-1,i}$ and $u_{t-1,i}$.}
\end{algorithm}
}

The algorithm operates as follows. First, in line $1$ we construct the second-order moment from the reservoir samples, where $\Pi_2(3)$ is the set of all $2$-permutations of $[3]$. Then we estimate $A_{t - 1}$ and $U_{t - 1}$ by eigendecomposition, and construct the whitening matrix $W_{t - 1}$ in line $2$. After whitening, in line $3$ we build the third-order tensor $T_{t - 1}$ from whitened words $((W_{t-1}^\top \mathbf{x}_z^{(l)})_{l = 1}^3)_{z \in \mathcal{S}_{t-1}}$, where $\Pi_3(3)$ is the set of all $3$-permutations of $[3]$. Then in line $4$ with the power iteration method \cite{anandkumar2014tensor}, we decompose $T_{t - 1}$ and get its eigenvalues $(\lambda_{t - 1, i})_{i = 1}^K$ and eigenvectors $(v_{t - 1, i})_{i = 1}^K$. Finally, in line $5$ we recover the parameters of the model, the probability of topics $(\omega_{t - 1, i})_{i = 1}^K$ and the conditional probability of words $(u_{t - 1, i})_{i = 1}^K$. After recovering the parameters, we update the set of reservoir samples from line $6$ to line $13$. We keep $m$ reservoir samples $\mathbf{x}_z$, $z \in [t-1]$. When $t \leq m$, the new observation $(\mathbf{x}_t^{(l)})_{l = 1}^3$ is added to the pool. When $t > m$, the new observation $(\mathbf{x}_t^{(l)})_{l = 1}^3$ replaces a random observation in the pool with probability $m / (t-1)$.

In $\spectralfpl$, we introduce reservoir sampling for computational efficiency reasons. Without reservoir sampling, the operations in lines $1$ and $3$ of \cref{alg:spectralmethodnoisy} would depend on $t$ because all past observations are used to construct $M_{2,t-1}$ and $T_{t-1}$. Besides, the whitening operation in line $3$ would depend on $t$ because all past observations are whitened by a matrix $W_{t - 1}$ that changes with $t$. With reservoir sampling, we approximate $M_{2,t-1}$, $T_{t-1}$, and $W_{t - 1}$ with $m$ reservoir samples. We discuss how to set $m$ in detail in Section \ref{sec:analysis_sub2}.

%% file: Analysis.tex
\newcommand{\condE}[2]{\mathbb{E} \left[#1 \,\middle|\, #2\right]}
\newcommand{\condEsub}[3]{\mathbb{E}_{#3} \! \left[#1 \,\middle|\, #2\right]}
\newcommand{\E}[1]{\mathbb{E} \left[#1\right]}
\newcommand{\Esub}[2]{\mathbb{E}_{#2} \! \left[#1\right]}

\section{Analysis}
\label{sec:analysis}

In this section, we bound the regret of $\spectralfpl$. In \cref{sec:analysis_sub1}, we analyze the regret of $\spectralfpl$ without reservoir sampling in the noise-free setting. In this setting, the agent knows $(P_z)_{z=1}^{t-1}$ at time $t$. The regret is due to not knowing $P_t$ at time $t$. In \cref{sec:analysis_sub2}, we analyze the regret of $\spectralfpl$ with reservoir sampling in the noise-free setting. In this setting, the agent knows $(P_z)_{z \in \mathcal{S}_{t-1}}$ at time $t$, which is a subset of $(P_z)_{z=1}^{t-1}$. In comparison to \cref{sec:analysis_sub1}, the additional regret is due to reservoir sampling. In \cref{sec:analysis_sub3}, we discuss the regret of $\spectralfpl$ with reservoir sampling in the noisy setting. In this setting, the agent approximates each distribution $P_z$ with its single empirical observation $(\mathbf{x}_z^{(l)})_{l = 1}^3$, for any $z \in \mathcal{S}_{t-1}$. In comparison to \cref{sec:analysis_sub2}, the additional regret is due to noisy observations.

\input{Analysis_sub1}

\input{Analysis_sub2}

\input{Analysis_sub3}

\subsection{Technical Lemmas}
\label{sec:technical lemmas}

\begin{lemma}
\label{lem:argmin}
Let $\ell_t(M) = \|\mathbf{x}_t^{(1)}\otimes\mathbf{x}_t^{(2)}\otimes\mathbf{x}_t^{(3)} -M\|_F^2$. Then
\begin{align}
  \bar{M}_{3, n} = \argmin_{M \in [0, 1]^{d \times d \times d}} \sum_{t=1}^{n}\mathbb{E}[\ell_t(M)]\,,
  \label{eq:argmin}
\end{align}
where $\bar{M}_{3, n}$ is defined in \eqref{eq:barmn}.
\end{lemma}
\begin{proof}
It is sufficient to show that
\begin{align}
  \label{eq:argminentry}
  \bar{M}_{3, n}(i,j,k) = \argmin_{y \in [0, 1]}\frac{1}{n}\sum_{t=1}^{n}\mathbb{E}[(\mathbf{x}_t^{(1)}\otimes\mathbf{x}_t^{(2)}\otimes\mathbf{x}_t^{(3)}(i,j,k) - y)^2]
\end{align}
for any $(i, j, k)$, where $\bar{M}_{3, n}(i,j,k)$ and $\mathbf{x}_t^{(1)}\otimes\mathbf{x}_t^{(2)}\otimes\mathbf{x}_t^{(3)}(i,j,k)$ are the $(i,j,k)$-th entries of tensors $\bar{M}_{3, n}$ and $\mathbf{x}_t^{(1)}\otimes\mathbf{x}_t^{(2)}\otimes\mathbf{x}_t^{(3)}$, respectively. To prove the claim, let $f(y) = \frac{1}{n}\sum_{t=1}^{n}\mathbb{E}[(\mathbf{x}_t^{(1)}\otimes\mathbf{x}_t^{(2)}\otimes\mathbf{x}_t^{(3)}(i,j,k) - y)^2]$. Then
\begin{align*}
  \frac{\partial}{\partial y} f(y) =
  2 y -
  \frac{2}{n}\sum_{t=1}^{n}\mathbb{E}[\mathbf{x}_t^{(1)}\otimes\mathbf{x}_t^{(2)}\otimes\mathbf{x}_t^{(3)}(i,j,k)]\,.
\end{align*}
Now we put the derivative equal to zero and get $y = \bar{M}_{3, n}(i,j,k)$. This concludes our proof.
\end{proof}

\begin{lemma}
\label{lem:c0} For any $n$, $\sum_{t=1}^n \mathbb{E}[\ell_t(\bar{M}_{3, n})]  \geq \sum_{t=1}^n \mathbb{E}[\ell_t(\bar{M}_{3, t})]$.
\end{lemma}
\begin{proof}
We prove this claim by induction. First, suppose that $n = 0$. Then trivially $\mathbb{E}[\ell_t(\bar{M}_{3,0})]  \geq \mathbb{E}[\ell_t(\bar{M}_{3,0})]$. Second, by induction hypothesis, we have that
\begin{align}
\label{eq:induction_assume}
\sum_{t=1}^{n-1} \mathbb{E}[\ell_t(\bar{M}_{3,n-1})]  \geq \sum_{t=1}^{n-1} \mathbb{E}[\ell_t(\bar{M}_{3, t})]\,.
\end{align} 
Then
\begin{align*}
  \sum_{t=1}^n \mathbb{E}[\ell_t(\bar{M}_{3, n})]
  & = \sum_{t=1}^{n-1} \mathbb{E}[\ell_t(\bar{M}_{3, n})] +  \mathbb{E}[\ell_n(\bar{M}_{3, n})] \\
  & \geq \sum_{t=1}^{n-1} \mathbb{E}[\ell_t(\bar{M}_{3,n-1})]  +  \mathbb{E}[\ell_n(\bar{M}_{3, n})]\\
  & \geq \sum_{t=1}^{n} \mathbb{E}[\ell_t(\bar{M}_{3, t})]\,,
\end{align*}
where the first inequality is from (\ref{eq:argmin}) and the second inequality is from (\ref{eq:induction_assume}). This concludes our proof.
\end{proof}

\begin{lemma}
\label{lem:tool}
For any tensors $M \in [0, 1]^{d\times d\times d}$ satisfying $\sum_{i,j,k=1}^d M(i, j, k) =  1$, and $M' \in [0, 1]^{d\times d\times d}$ satisfying $\sum_{i,j,k=1}^d M'(i, j, k) =  1$, we have
\begin{align*}
  \ell_t(M) - \ell_t(M' ) \leq 4 \| M - M'\|_F.
\end{align*}
\end{lemma}
\begin{proof}
The proof follows from elementary algebra
\begin{align*}
  & \ell_t(M) - \ell_t(M' )\\
  & \quad = (\ell_t^{\frac{1}{2}}(M) + \ell_t^{\frac{1}{2}}(M' )) (\ell_t^{\frac{1}{2}}(M) - \ell_t^{\frac{1}{2}}(M' ))\\
  & \quad \leq (\ell_t^{\frac{1}{2}}(M) + \ell_t^{\frac{1}{2}}(M'))  \| M - M' \|_F \\
  & \quad = \left(\| \mathbf{x}_t^{(1)}\otimes\mathbf{x}_t^{(2)}\otimes\mathbf{x}_t^{(3)} -M \|_F + \| \mathbf{x}_t^{(1)}\otimes\mathbf{x}_t^{(2)}\otimes\mathbf{x}_t^{(3)} -M' \|_F\right) \| M - M' \|_F \\
  & \quad \leq \left(2 \| \mathbf{x}_t^{(1)}\otimes\mathbf{x}_t^{(2)}\otimes\mathbf{x}_t^{(3)}\|_F + \|M \|_F +\| M' \|_F\right) \| M - M' \|_F \\
  & \quad = \left(2 + \|M \|_F +\| M' \|_F\right) \| M - M' \|_F \\
  & \quad \leq 4 \|M - M' \|_F\,.
\end{align*}
The first equality is from $\alpha^2 -\beta^2 = (\alpha+\beta)(\alpha-\beta)$. The first inequality is from the reverse triangle inequality.  The second inequality is from the triangle inequality. The third equality is from the fact that only one entry of $\mathbf{x}_t^{(1)}\otimes\mathbf{x}_t^{(2)}\otimes\mathbf{x}_t^{(3)}$ is $1$ and all the rest are $0$, by the definition of $(\mathbf{x}_t^{(l)})_{l = 1}^3$ in Section \ref{sec:setting}. The third inequality is from
$\|M \|_F = \sqrt{\sum_{i,j,k=1}^d M(i, j, k)^2} \leq \sqrt{\sum_{i,j,k=1}^d |M(i, j, k)|} =  1$, and $\|M' \|_F = \sqrt{\sum_{i,j,k=1}^d M'(i, j, k)^2} \leq \sqrt{\sum_{i,j,k=1}^d |M'(i, j, k)|} =  1$, which follow the fact that tensors $M$ and $M'$ represent distributions with all entries in $[0, 1]$ and summing up to $1$. 
\end{proof}

%% file: Analysis_sub1.tex

\subsection{$\spectralfpl$ without Reservoir Sampling in Noise-Free Setting}
\label{sec:analysis_sub1}

We first analyze an idealized variant of $\spectralfpl$, where the agent knows $(P_z)_{z=1}^{t-1}$ at time $t$. In this setting, the algorithm is similar to \cref{alg:spectralmethodnoisy}, except that lines $1$ and $3$ are replaced by 
\begin{align*}
\bar{M}_{2,t-1} &= \frac{1}{t-1} \sum_{z =1}^{t-1} \mathbb{E}[\mathbf{x}_z^{(1)} \otimes \mathbf{x}_z^{(2)}]\,, \\
\bar{T}_{t-1} &= \frac{1}{t-1} \sum_{z =1}^{t-1} \mathbb{E}[\bar{W}_{t-1}^\top \mathbf{x}_z^{(1)} \otimes \bar{W}_{t-1}^\top \mathbf{x}_z^{(2)} \otimes \bar{W}_{t-1}^\top \mathbf{x}_z^{(3)}]\,.
\end{align*}
We denote by $\bar{W}_{t-1}$ the corresponding whitening matrix in line $2$, and by $\bar{\omega}_{t-1,i}$ and $\bar{u}_{t-1,i}$ the estimated model parameters. Note that in the noise free setting, the power iteration method in line $4$ is exact, and therefore
\begin{align*}
  \hat{M}_{3, t-1} = \sum_{i=1}^K \bar{\omega}_{t-1,i} \, \bar{u}_{t-1,i} \otimes \bar{u}_{t-1,i} \otimes \bar{u}_{t-1,i} = \bar{M}_{3, t-1}
\end{align*}
at any time $t$, according to \eqref{eq:decomposeM}. The main result of this section is stated below.

\begin{theorem}
\label{thm:noise-free regret} Let $\hat{M}_{3, t-1} = \bar{M}_{3, t - 1}$ at all times $t \in [n]$. Then
\begin{align*}
  R(n) \leq 4 \sqrt{d^3} \log n\,.
\end{align*}
\end{theorem}
\begin{proof}
From \cref{lem:c0} in \cref{sec:technical lemmas}, $\sum_{t=1}^n \mathbb{E}[\ell_t(\bar{M}_{3, n})]  \geq \sum_{t=1}^n \mathbb{E}[\ell_t(\bar{M}_{3, t})]$. Now note that $\hat{M}_{3, t-1} = \bar{M}_{3, t-1}$ at any time $t$, and therefore
\begin{align*}
  R(n) =
  \sum_{t=1}^n \mathbb{E}[\ell_t(\bar{M}_{3, t-1}) - \ell_t(\bar{M}_{3, n})] \leq
  \sum_{t=1}^n \mathbb{E}[\ell_t(\bar{M}_{3, t-1}) - \ell_t(\bar{M}_{3, t})]\,.
\end{align*}
At any time $t$ and for any $\mathbf{x}_t$,
\begin{align*}
  \ell_t(\bar{M}_{3, t-1}) - \ell_t(\bar{M}_{3, t})
  & \leq 4 \| \bar{M}_{3, t-1} - \bar{M}_{3, t}\|_F \\
  & = 4 \left\| \frac{1}{t-1} \sum_{t' = 1}^{t-1} P_{t'}- \frac{1}{t} \sum_{t' = 1}^{t} P_{t'}\right\|_F \\
  & = \frac{4}{t} \left\| \frac{1}{t - 1} \sum_{t' = 1}^{t-1} P_{t'} - P_t \right\|_F \\
  & \leq \frac{4 \sqrt{d^3}}{t}\,,
\end{align*}
where the first inequality is by \cref{lem:tool} and the second inequality is from the fact that all entries of $P_t$ are in $[0, 1]$ at any time $t \in [n]$.
\begin{align*}
  R(n) \leq
  \sum_{t = 1}^n \frac{4 \sqrt{d^3}}{t} \leq
  4 \sqrt{d^3} \log n\,.
\end{align*}
This concludes our proof.
\end{proof}

%% file: Analysis_sub2.tex

\subsection{$\spectralfpl$ with Reservoir Sampling in Noise-Free Setting}
\label{sec:analysis_sub2}

We further analyze $\spectralfpl$ with reservoir sampling in the noise-free setting. As discussed in \cref{sec:spectralleader}, without reservoir sampling, the construction time of the decomposed tensor at time $t$ would grow linearly with $t$, which is undesirable. In this setting, the algorithm is similar to \cref{alg:spectralmethodnoisy}, except that lines $1$ and $3$ are replaced by
\begin{align*}
  \tilde{M}_{2,t-1} & =
  \frac{1}{|\mathcal{S}_{t-1}| } \mathbb{E}[\sum_{z \in \mathcal{S}_{t-1}} \mathbf{x}_z^{(1)} \otimes \mathbf{x}_z^{(2)}]\,, \\
  \tilde{T}_{t-1} & =
  \frac{1}{|\mathcal{S}_{t-1}|} \mathbb{E}[\sum_{z \in \mathcal{S}_{t-1}} \tilde{W}_{t-1}^\top \mathbf{x}_z^{(1)} \otimes \tilde{W}_{t-1}^\top \mathbf{x}_z^{(2)} \otimes \tilde{W}_{t-1}^\top \mathbf{x}_z^{(3)}]\,,
\end{align*}
where $\mathcal{S}_{t-1}$ are indices of the reservoir samples at time $t$. We denote by $\tilde{W}_{t-1}$ the corresponding whitening matrix in line $2$, and by $\tilde{\omega}_{t-1,i}$ and $\tilde{u}_{t-1,i}$ the estimated model parameters. As in \cref{sec:analysis_sub1}, the power iteration method in line $4$ is exact, and therefore
\begin{align*}
  \hat{M}_{3, t-1} =
  \sum_{i=1}^K \tilde{\omega}_{t-1,i} \, \bar{u}_{t-1,i} \otimes \tilde{u}_{t-1,i} \otimes \tilde{u}_{t-1,i} =
  \tilde{M}_{3, t-1}
\end{align*}
at any time $t$. The main result of this section is stated below.

\begin{theorem}
\label{thm:noise-free-sample regret} Let all corresponding entries of $\tilde{M}_{3, t-1}$ and $\bar{M}_{3, t-1}$ be close with a high probability,
\begin{align}
  P(\exists t, i, j, k: |\tilde{M}_{3, t-1}(i, j, k) - \bar{M}_{3, t-1}(i, j, k)| \geq \epsilon) = \delta
  \label{eq:assumption}
\end{align}
for some small $\epsilon \in [0, 1]$ and $\delta \in [0, 1]$. Let $\hat{M}_{3, t - 1} = \tilde{M}_{3, t - 1}$ at all times $t \in [n]$. Then
\begin{align*}
  R(n) \leq 4\sqrt{d^3} \epsilon n +4\sqrt{d^3} \delta  n + 4\sqrt{d^3} \log n\,.
\end{align*}
\end{theorem}
\begin{proof}
From the definition of $R(n)$ and the bound in \cref{thm:noise-free regret},
\begin{align*}
  R(n)
  & = \sum_{t=1}^n \mathbb{E}[\ell_t(\tilde{M}_{3, t-1}) - \ell_t(\bar{M}_{3, t-1})] + \sum_{t=1}^n \mathbb{E}[\ell_t(\bar{M}_{3, t-1}) - \ell_t(\bar{M}_{3, n})]\\
  & \leq \sum_{t=1}^n \mathbb{E}[\ell_t(\tilde{M}_{3, t-1}) - \ell_t(\bar{M}_{3, t-1})] +4\sqrt{d^3} \log n\,.
\end{align*}
We bound the first term above as follows. Suppose that the event in \eqref{eq:assumption} does not happen. Then $\ell_t(\tilde{M}_{3, t-1}) - \ell_t(\bar{M}_{3, t-1}) \leq 4\sqrt{d^3} \epsilon$, from \cref{lem:tool} and the fact that all corresponding entries of $\tilde{M}_{3, t-1}$ and $\bar{M}_{3, t-1}$ are $\epsilon$ close. Now suppose that the event in \eqref{eq:assumption} happens. Then $\ell_t(\tilde{M}_{3, t-1}) - \ell_t(\bar{M}_{3, t-1}) \leq 4\sqrt{d^3}$, from \cref{lem:tool} and the fact all entries of $\tilde{M}_{3, t-1}$ and $\bar{M}_{3, t-1}$ are in $[0, 1]$. Finally, note that the event in \eqref{eq:assumption} happens with probability $\delta$. Now we chain all inequalities and get that $R(n) \leq 4\sqrt{d^3} \epsilon n +4\sqrt{d^3} \delta  n + 4\sqrt{d^3} \log n$.
\end{proof}

Note that the reservoir at time $t$, $\mathcal{S}_{t-1} \in [t - 1]$, is a random sample of size $m$ for any $t > m + 1$. Therefore, from Hoeffding's inequality \cite{hoeffding1963probability} and the union bound, we get that
\begin{align*}
  \delta
  & = P(\exists t, i, j, k: |\tilde{M}_{3, t-1}(i, j, k) - \bar{M}_{3, t-1}(i, j, k)| \geq \epsilon) \\
  & \leq 2 \sum_{t = m + 2}^n d^3 \exp[-2 \epsilon^2 m] \\
  & \leq 2 d^3 n \exp[-2 \epsilon^2 m]\,.
\end{align*}
In addition, let the size of the reservoir be $m = \epsilon^{-2} \log(d^3 n)$. Then the regret bound in \cref{thm:noise-free-sample regret} simplifies to
\begin{align*}
  R(n) < 4\sqrt{d^3} \epsilon n + 4\sqrt{d^3} \log n + 8\,.
\end{align*}
The above bound can be sublinear in $n$ only if $\epsilon = o(1)$. Moreover, the definition of $m$ and $m \leq n$ imply that $\epsilon \geq \sqrt{\log(d^3 n) / n}$. As a result of these constraints, the range of reasonable values for $\epsilon$ is $[\sqrt{\log(d^3 n) / n}, o(1))$.

For any $\epsilon \in [\sqrt{\log(d^3 n) / n}, o(1))$, the regret $R(n)$ is sublinear in $n$, where $\epsilon$ is a tunable parameter. At lower values of $\epsilon$, $R(n) = O(\sqrt{n})$ but the reservoir size approaches $n$. At higher values of $\epsilon$, the reservoir size is $O(\log{n})$ but $R(n)$ approaches $n$.

%% file: Analysis_sub3.tex

\subsection{$\spectralfpl$ with Reservoir Sampling in Noisy Setting}
\label{sec:analysis_sub3}

Finally, we discuss the regret of $\spectralfpl$ with reservoir sampling in the noisy setting. In this setting, the analyzed algorithm is \cref{alg:spectralmethodnoisy}. The predicted distribution at time $t$ is $\hat{M}_{3, t - 1} = \sum_{i = 1}^K \omega_{t - 1, i} \, u_{t - 1, i} \otimes u_{t - 1, i} \otimes u_{t - 1, i}$.

From the definition of $R(n)$ and the discussion in \cref{sec:analysis_sub3}, $R(n)$ can be decomposed and bounded from above as
\begin{align}
  R(n) \leq
  \sum_{t = 1}^n \mathbb{E}[\ell_t(\hat{M}_{3, t-1}) - \ell_t(\tilde{M}_{3, t-1})] + 4 \sqrt{d^3} \epsilon n + 4 \sqrt{d^3} \log n + 8
  \label{eq:final bound}
\end{align}
when the size of the reservoir is $m = \epsilon^{-2} \log(d^3 n)$.

Suppose that $m \to \infty$ as $n \to \infty$, for instance by setting $\epsilon = n^{- \frac{1}{4}}$. Under this assumption, $M_{2, t - 1}$ in $\spectralfpl$ approaches $\tilde{M}_{2, t - 1}$ (\cref{sec:analysis_sub2}) because $M_{2, t - 1}$ is an empirical estimator of $\tilde{M}_{2, t - 1}$ on $m$ observations. By Weyl's and Davis-Kahan theorems \cite{weyl1912asymptotische,davis1970rotation}, the eigenvalues and eigenvectors of $M_{2, t - 1}$ approach those of $\tilde{M}_{2, t - 1}$ as $m \to \infty$, and thus the whitening matrix $W_{t - 1}$ in $\spectralfpl$ approaches $\tilde{W}_{t - 1}$ (\cref{sec:analysis_sub2}). Since $T_{t - 1}$ in $\spectralfpl$ is an empirical estimator of $\tilde{T}_{t - 1}$ (\cref{sec:analysis_sub2}) on $m$ whitened observations, where $W_{t - 1} \to \tilde{W}_{t - 1}$, $T_{t - 1} \to \tilde{T}_{t - 1}$ as $m \to \infty$. By Theorem 5.1 of Anandkumar \etal~\shortcite{anandkumar2014tensor}, the eigenvalues and eigenvectors of $T_{t - 1}$ approach those of $\tilde{T}_{t - 1}$ as $T_{t - 1} \to \tilde{T}_{t - 1}$. This implies that $\hat{M}_{3, t - 1} \to\tilde{M}_{3, t - 1}$, as all quantities in the definitions of $\hat{M}_{3, t - 1}$ and $\tilde{M}_{3, t - 1}$ approach each other as $m \to \infty$. Therefore, $\lim_{n \to \infty} \lim_{t \to n} (\ell_t(\hat{M}_{3, t - 1}) - \ell_t(\tilde{M}_{3, t - 1})) = 0$ and the regret bound in \eqref{eq:final bound} is $o(n)$, sublinear in $n$, as $n \to \infty$.

%% file: Experiments.tex

\begin{figure}[hp]
  \centering
  \begin{subfigure}[b]{0.40\textwidth}
  \includegraphics[width=1\textwidth]{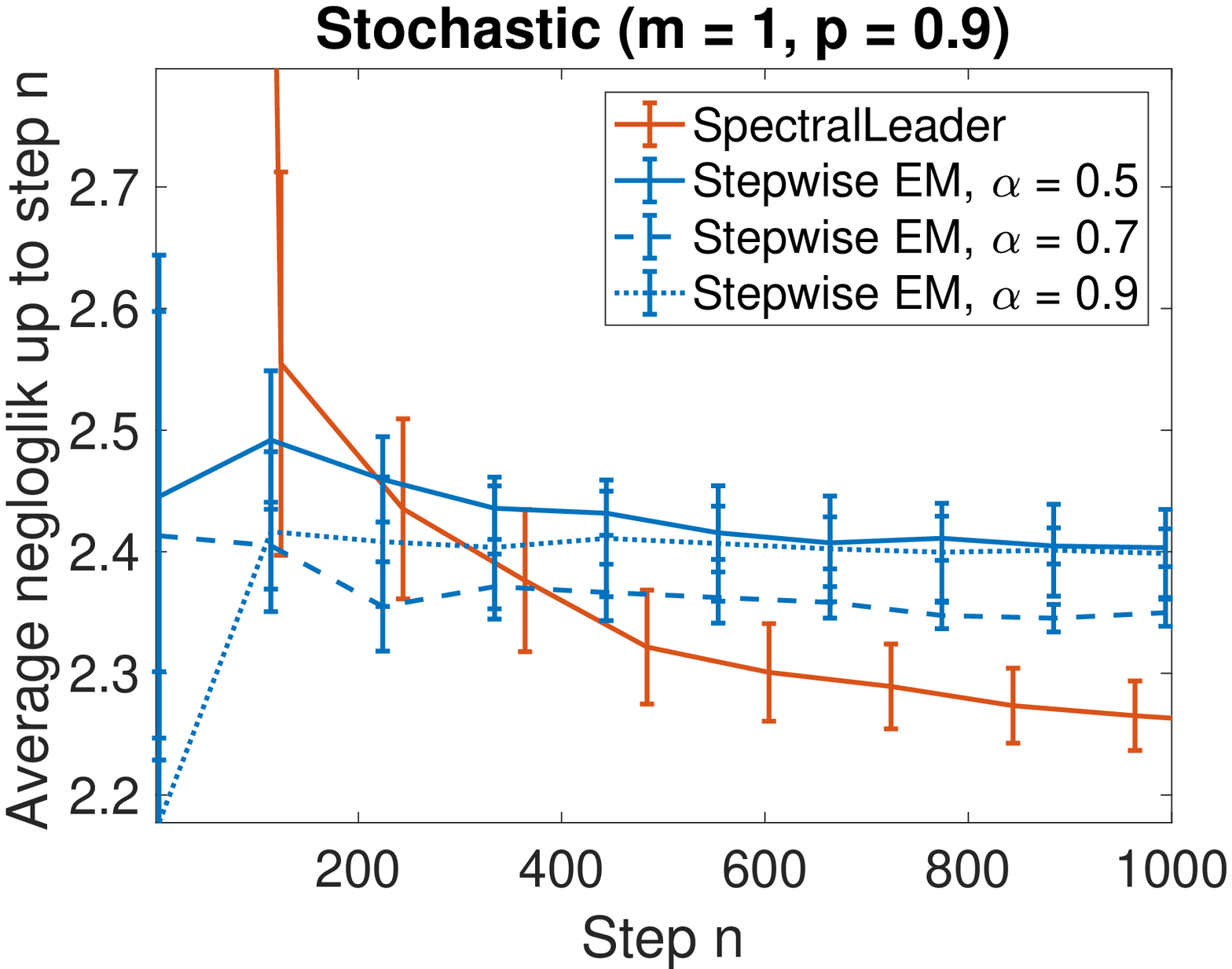}
   \caption{}
    \label{fig:synstoa}
  \end{subfigure} 
  \begin{subfigure}[b]{0.402\textwidth}
\includegraphics[width=1\textwidth]{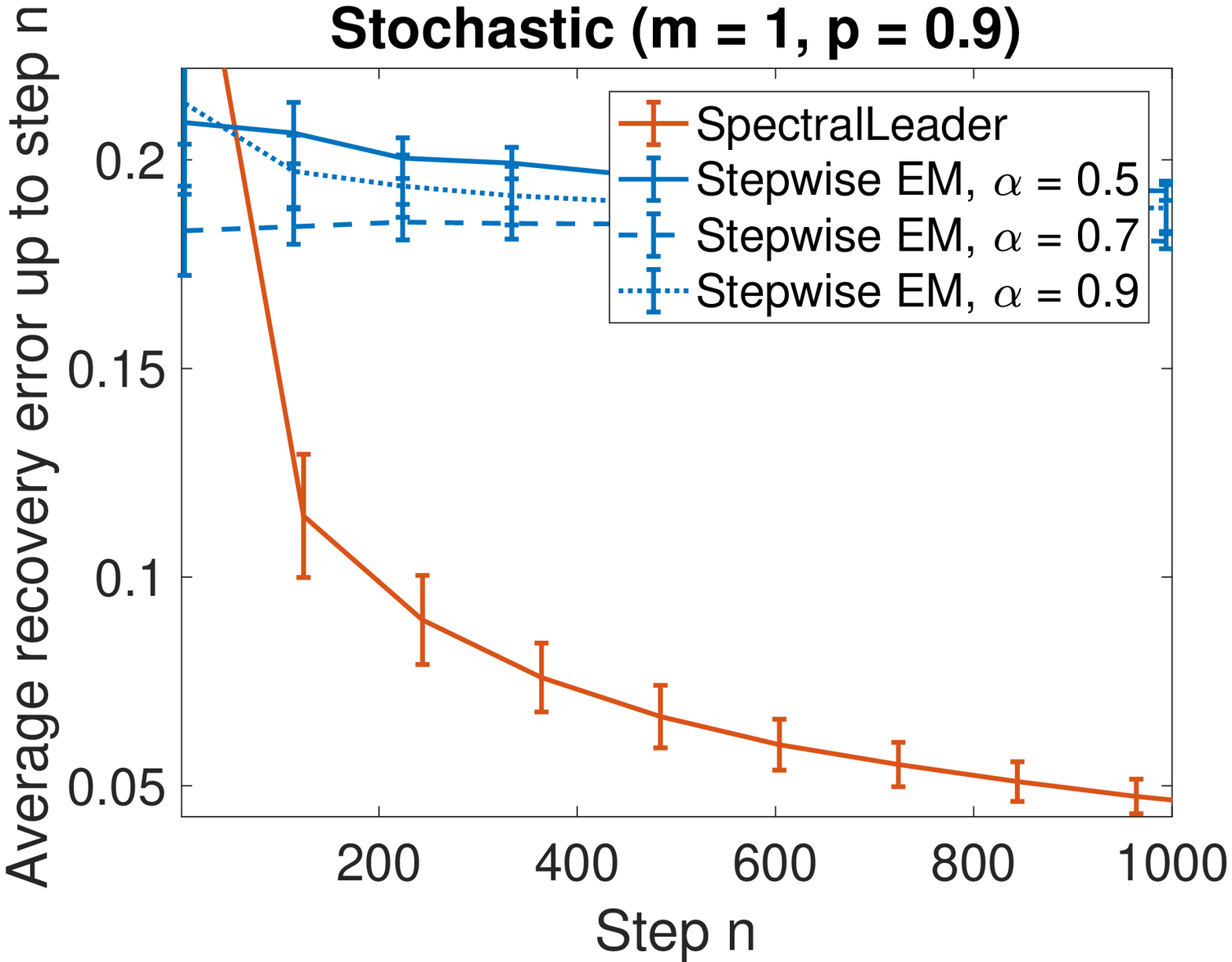}
  \caption{}
    \label{fig:synstoe}
  \end{subfigure} \\
  \begin{subfigure}[b]{0.40\textwidth}
 \includegraphics[width=1\textwidth]{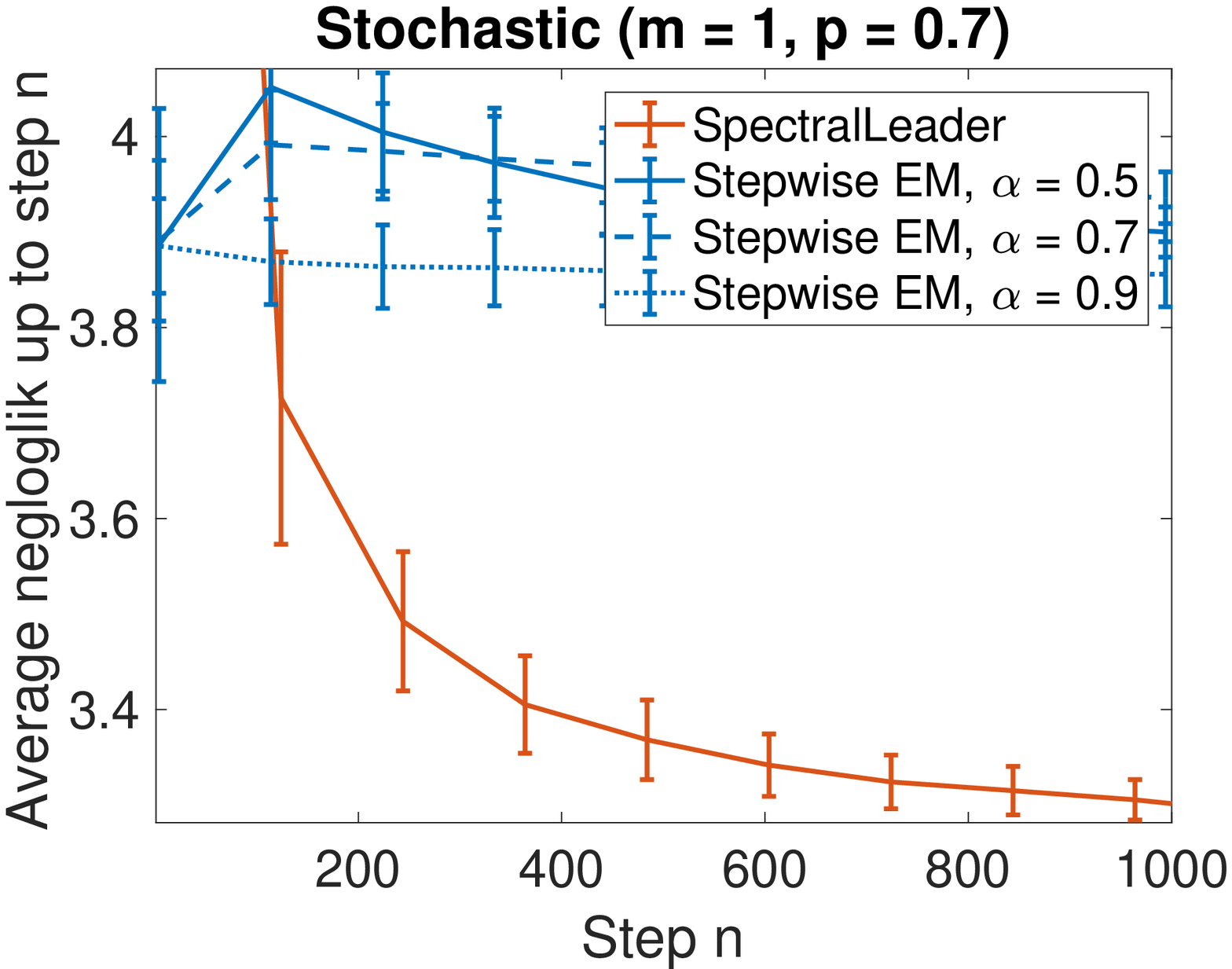}
   \caption{}
    \label{fig:synstoc}
  \end{subfigure} 
  \begin{subfigure}[b]{0.402\textwidth}
 \includegraphics[width=1\textwidth]{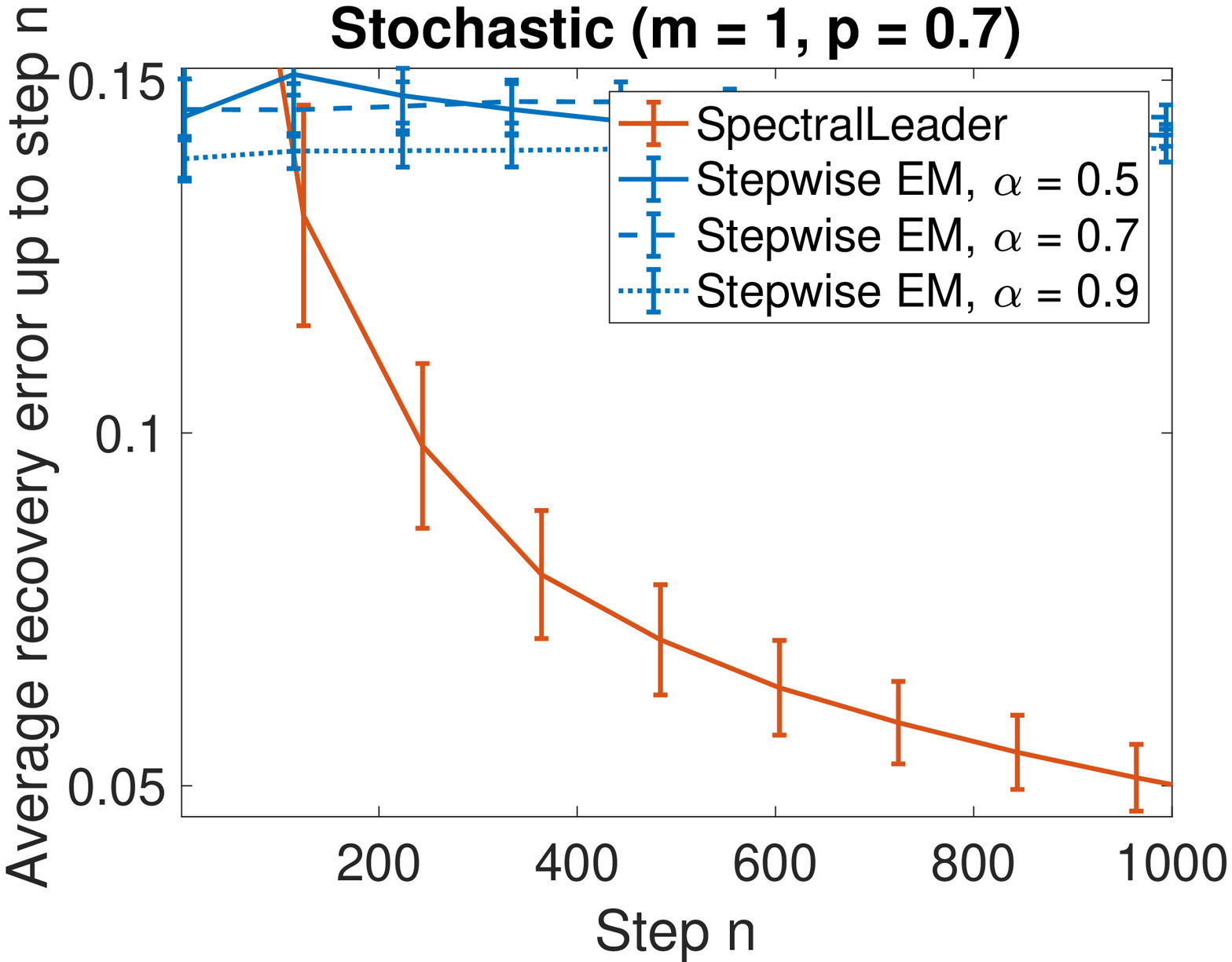}
   \caption{}
    \label{fig:synstog}
  \end{subfigure} \\
  \begin{subfigure}[b]{0.40\textwidth}
 \includegraphics[width=1\textwidth]{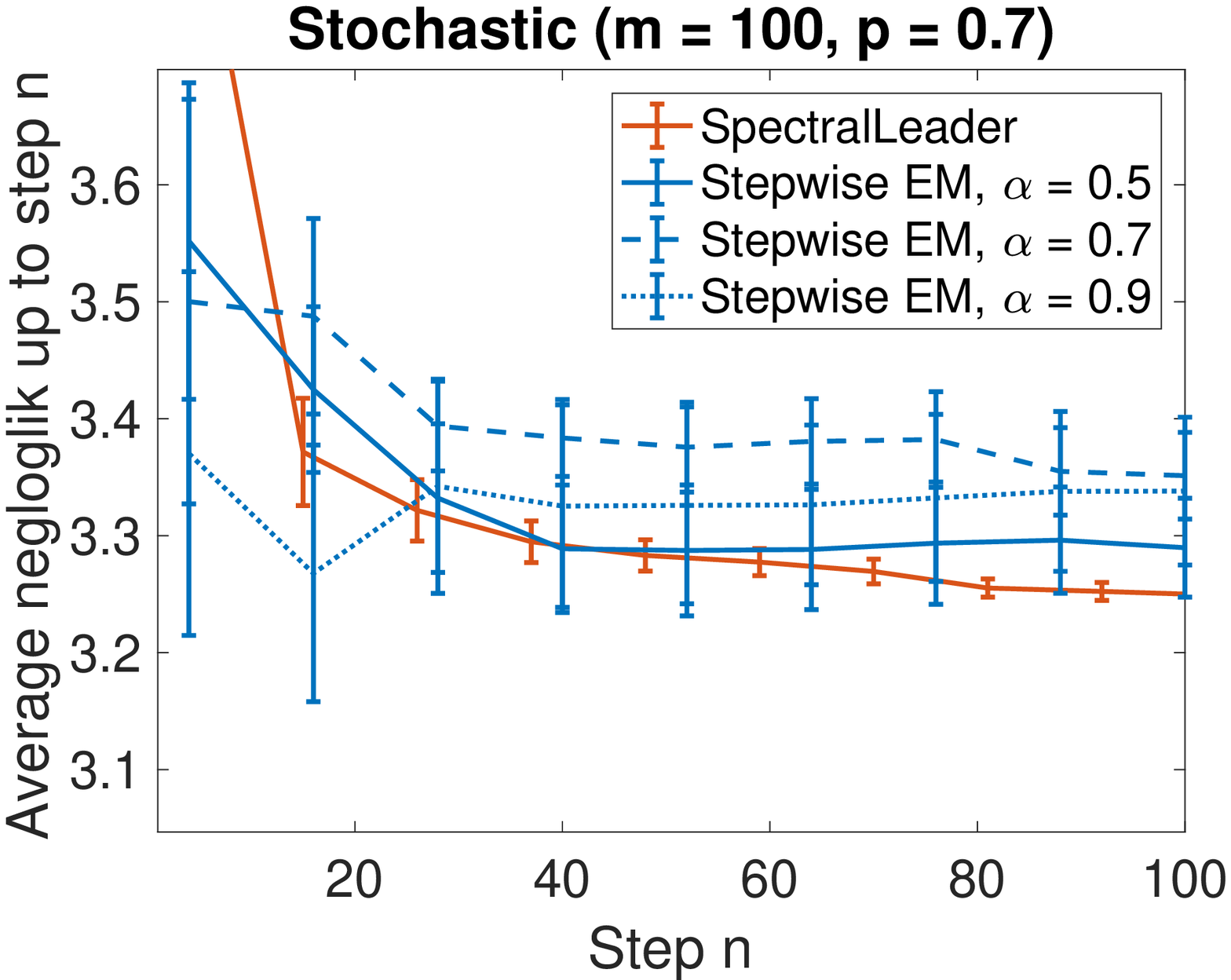}  \caption{}
    \label{fig:synstod}
  \end{subfigure} 
  \begin{subfigure}[b]{0.402\textwidth}
 \includegraphics[width=1\textwidth]{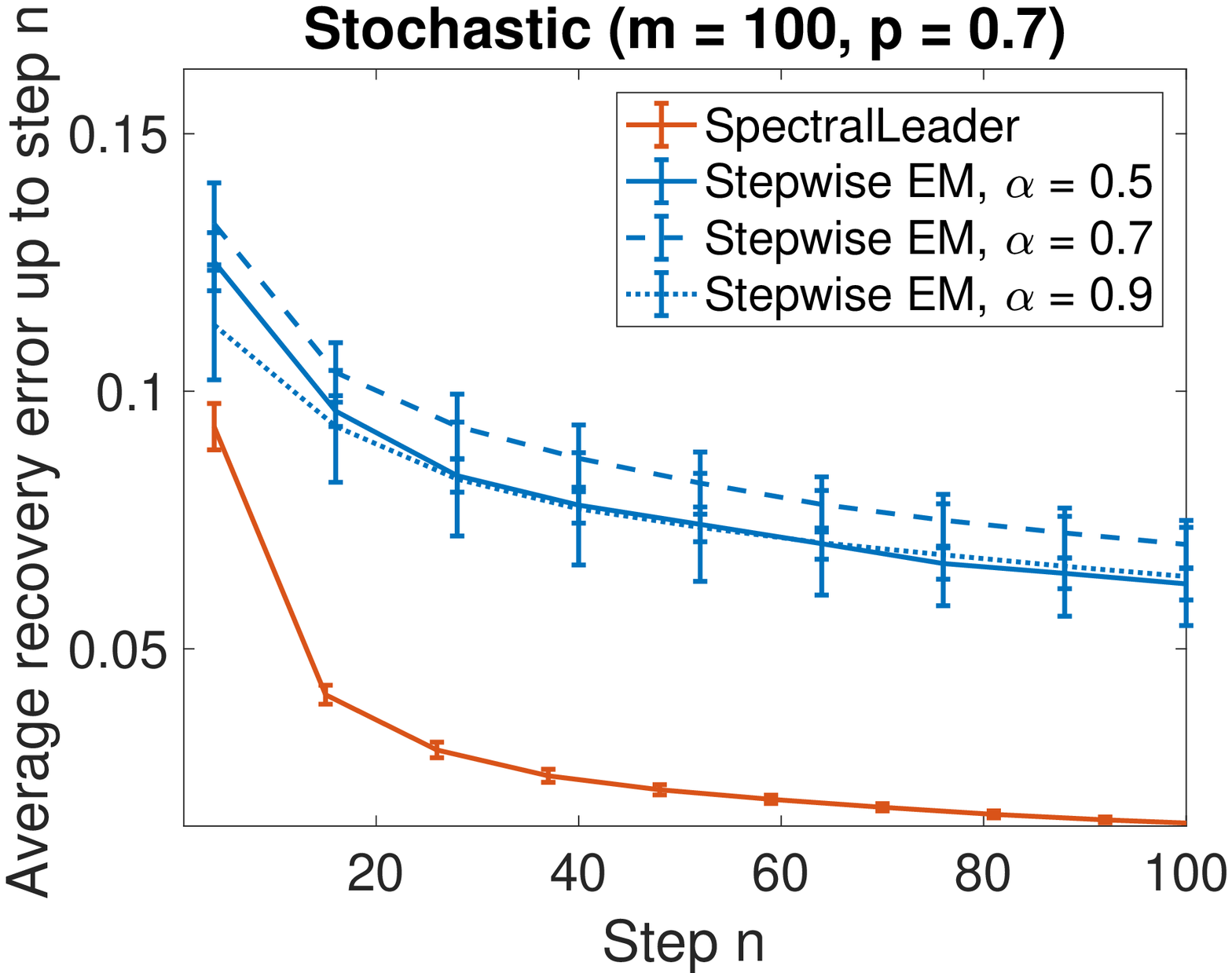}
   \caption{}
    \label{fig:synstoh}
  \end{subfigure} \\
  \begin{subfigure}[b]{0.40\textwidth}
 \includegraphics[width=1\textwidth]{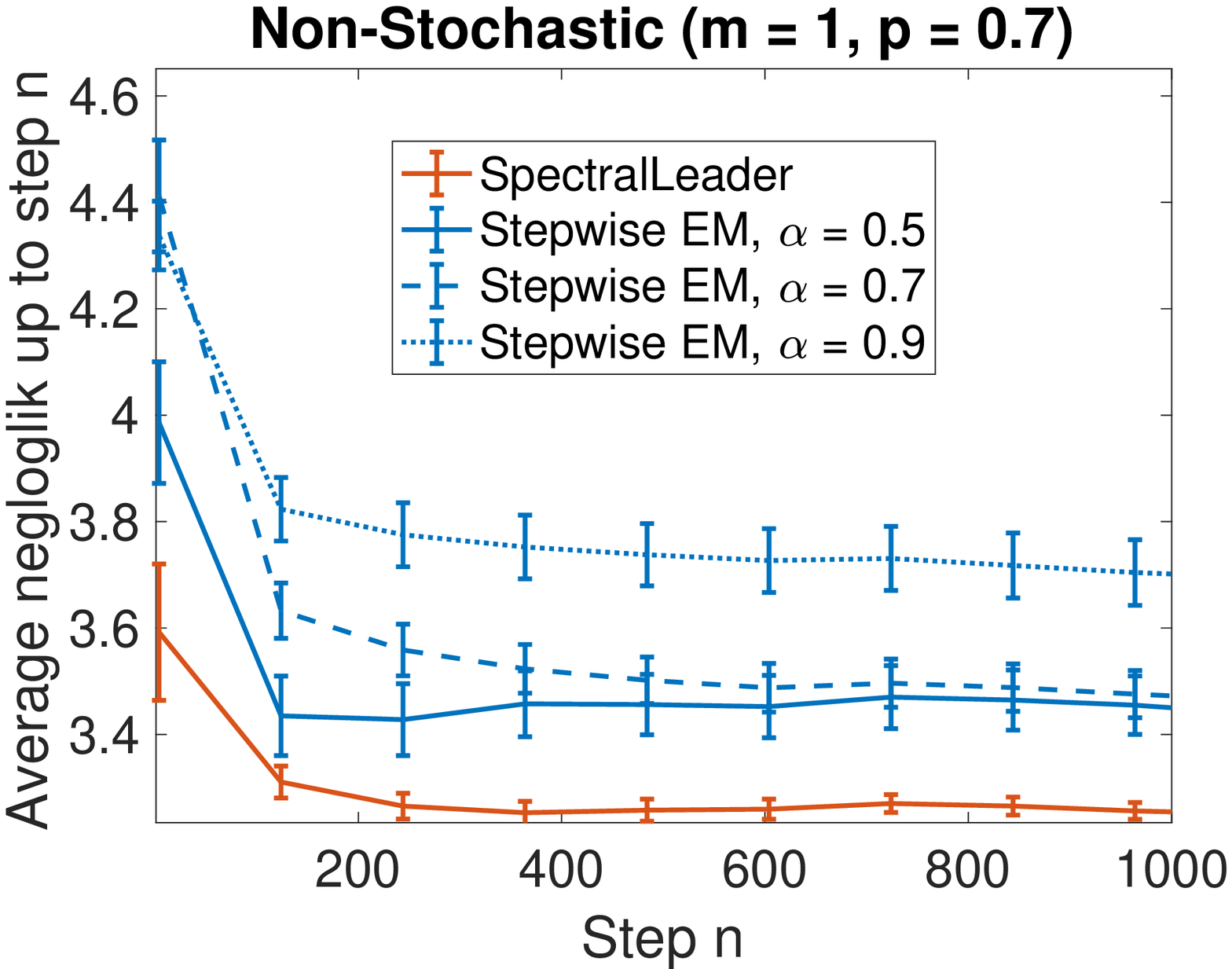}  \caption{}
    \label{fig:non1}
  \end{subfigure} 
  \begin{subfigure}[b]{0.402\textwidth}
 \includegraphics[width=1\textwidth]{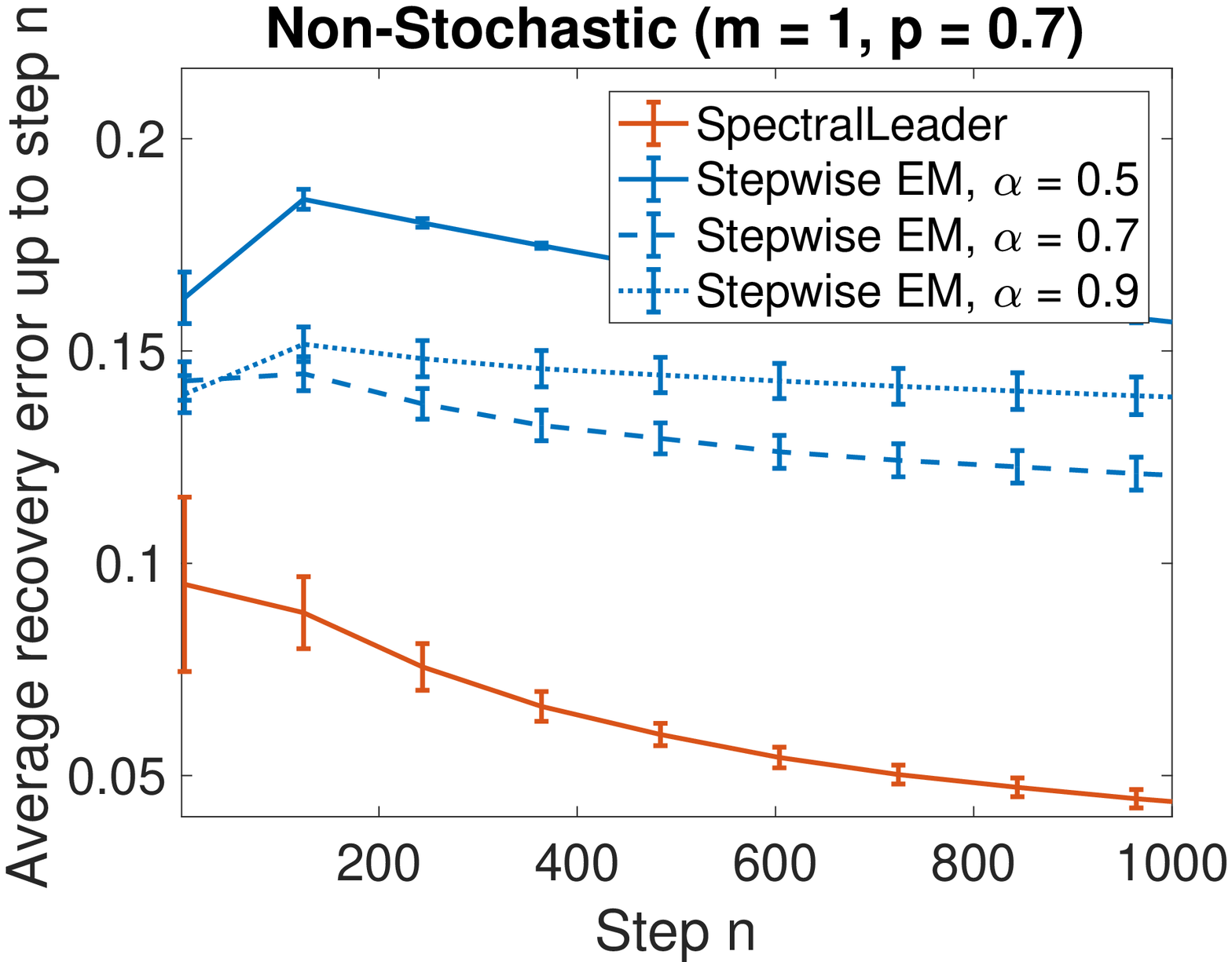}
   \caption{}
    \label{fig:non2}
  \end{subfigure} 
\caption{The evaluations in synthetic problems. The first column shows the results under the metric $\mathcal{L}_{n}^{(1)}$ and the second column shows the results under the metric $\mathcal{L}_{n}^{(2)}$.} 
 \label{fig:synsto1}
\end{figure}

\section{Experiments}

In this section, we evaluate $\spectralfpl$ and compare it with stepwise EM \cite{cappe2009line}. We experiment with both stochastic and non-stochastic synthetic problems, as well as with two real-world problems. 

Our chosen baseline is stepwise EM \cite{cappe2009line}, an online EM algorithm. We choose this baseline as it outperforms other online EM algorithms, such as incremental EM \cite{liang2009online}. Stepwise EM has two key tuning parameters: the step-size reduction power $\alpha$ and the mini-batch size $m$  \cite{liang2009online,cappe2009line}. 
The smaller the $\alpha$, the faster the old sufficient statistics are forgotten. The mini-batch size $m$ is the number of documents to calculate the sufficient statistics for each update of stepwise EM. In the following experiments, we compared $\spectralfpl$ to stepwise EM with varying $\alpha$ and $m$.

All compared algorithms are evaluated by their models at time $t$, $\theta_{t-1} = ((\omega_{t - 1, i})_{i = 1}^K, (u_{t - 1, i})_{i = 1}^K)$, which are learned from the first $t - 1$ steps. We report two metrics: \emph{average negative predictive log-likelihood up to step $n$}, $\mathcal{L}_{n}^{(1)} = \frac{1}{n}\sum_{t=2}^n  \left(-\log \sum_{i=1}^K P_{\theta_{t-1}}(C=i) \prod_{l=1}^L P_{\theta_{t-1}}(\mathbf{x}=\mathbf{x}_t^{(l)} \mid C=i) \right)$, where $L$ is the number of observed words in each document; and \emph{average recovery error up to step $n$}, $\mathcal{L}_{n}^{(2)} =  \frac{1}{n}\sum_{t=2}^n \|M_{3,\ast} - \hat{M}_{3,t-1}\|_F^2$. The latter metric is the average difference between the distribution in hindsight $M_{3,\ast}$ and the predicted distribution $\hat{M}_{3,t-1}$ at time $t$, and measures the parameter reconstruction error. Specifically, $M_{3,\ast} = \sum_{i=1}^K \omega_{*, i} u_{*, i} \otimes u_{*, i} \otimes u_{*, i}$ and $\hat{M}_{3,t-1} = \sum_{i=1}^K \omega_{t-1, i} u_{t-1, i} \otimes u_{t-1, i} \otimes u_{t-1, i}$, where $\theta^\ast = ((\omega_{*, i})_{i = 1}^K, (u_{*, i})_{i = 1}^K)$ are the parameters of the unknown model. In synthetic problems, we know $\theta^\ast$. In real-world problems, we learn $\theta^\ast$ by the spectral method because we have all data in advance. The recovery error is related to the regret, through the relation of our loss function and the Frobenius norm in \cref{lem:tool}. Note that EM in the offline setting minimizes the negative log-likelihood, while the spectral method in the offline setting minimizes the recovery error of tensors. All reported results are averaged over $10$ runs.

\subsection{Synthetic Stochastic Setting}
\label{sec:synsto}
We compare $\spectralfpl$ with stepwise EM on two synthetic problems in the stochastic setting. In this setting, the topic of the document at all times $t$ is sampled i.i.d. from a fixed distribution. This setting represents a scenario where the sequence of topics is not correlated. The number of distinct topics is $K = 3$, the vocabulary size is $d = 3$, and each document has $3$ observed words. In practice, some topics are more popular than others. Therefore, we sample topics as follows. At each time, the topic is randomly sampled from the distribution where $P(C = 1) = 0.15$, $P(C = 2) = 0.35$, and $P(C = 3) = 0.5$. Given the topic, the conditional probability of words is $P(\mathbf{x} = e_i|C=j) = p$ when $i=j$, and $P(\mathbf{x} = e_i|C=j) = \frac{1-p}{2}$ when $i\ne j$. With smaller $p$, the conditional distribution of words given different topic becomes similar, and the difficulty of distinguishing different topics increases. In Section \ref{sec:synsto} and \ref{sec:synnonsto}, we define the hard problem as the synthetic problem with $p = 0.7$ and the easy problem as the synthetic problem with $p = 0.9$. For $m = 1$, we evaluate on the easy problem and the hard problem. For $m = 100$, we further focus on the hard problem. 
We show the results before the different methods converge: for $m = 1$, we report results before $n = 1000$, and for $m = 100$ we report both results before $n = 100$. 

Our results are reported in Figure \ref{fig:synsto1}. We observe three trends. First, under the metric $\mathcal{L}_{n}^{(1)}$, 
stepwise EM is very sensitive to its parameters $\alpha$ and $m$, while $\spectralfpl$ is competitive or even better, compared to the stepwise EM with its best $\alpha$ and $m$. For example, the best $\alpha$ is $0.7$ in Figure \ref{fig:synstoa}, and the best $\alpha$ is $0.9$ in Figure \ref{fig:synstoc}. Even for the same problem with different $m$, the best $\alpha$ is different: the best $\alpha$ is $0.9$ in Figure \ref{fig:synstoc}, while the best $\alpha$ is $0.5$ in Figure \ref{fig:synstod}. In all cases, $\spectralfpl$ performs the best. Second, similar to \cite{liang2009online}, stepwise EM improves when the mini-batch size increases to $m = 100$. But $\spectralfpl$ still performs better compared to stepwise EM with its best $\alpha$. Third, $\spectralfpl$ performs much better than stepwise EM under the metric $\mathcal{L}_{n}^{(2)}$. These results indicate that a careful grid search of $\alpha$ and $m$ is usually needed to optimize stepwise EM. In contrast, $\spectralfpl$ is very competitive without any parameter tuning. Note that grid search in the online setting is nearly impossible, since the future data are unknown in advance.

\begin{figure*}[hp]
  \centering
   \begin{subfigure}[b]{0.40\textwidth}
  \includegraphics[width=1\textwidth]{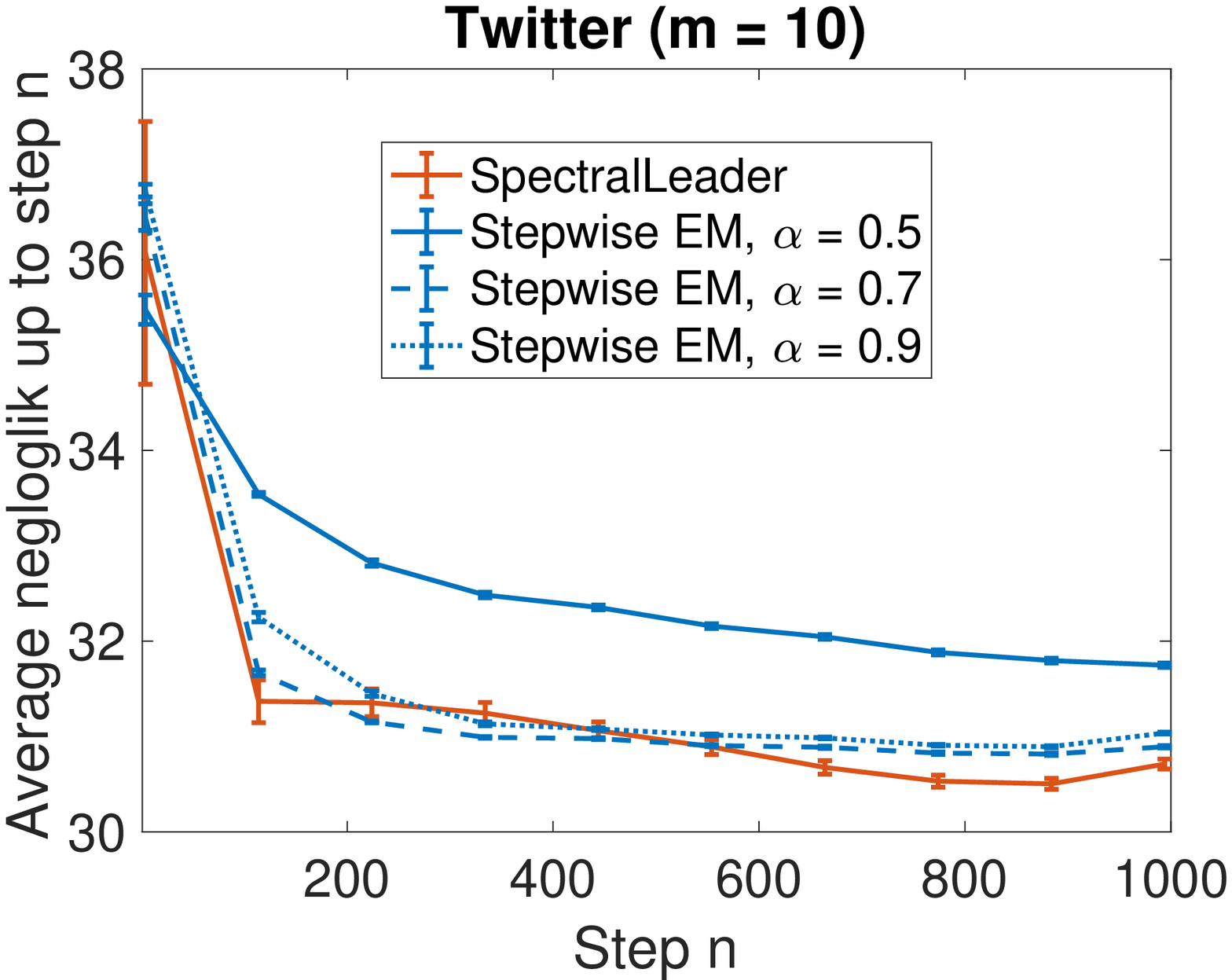}
   \caption{}
    \label{fig:real_a}
  \end{subfigure} 
  \begin{subfigure}[b]{0.415\textwidth}
\includegraphics[width=1\textwidth]{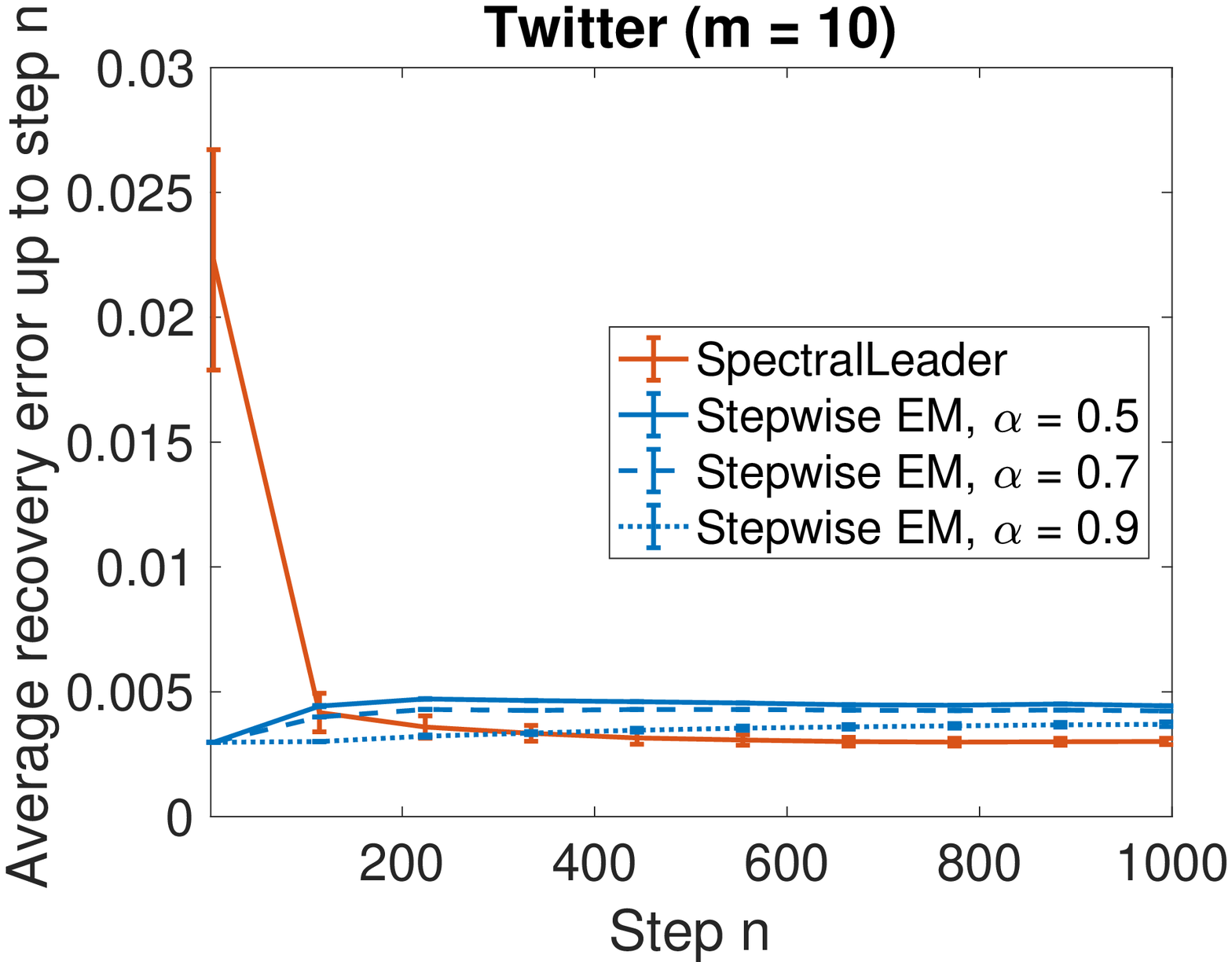}
  \caption{}
    \label{fig:real_e}
  \end{subfigure} \\
  \begin{subfigure}[b]{0.40\textwidth}
 \includegraphics[width=1\textwidth]{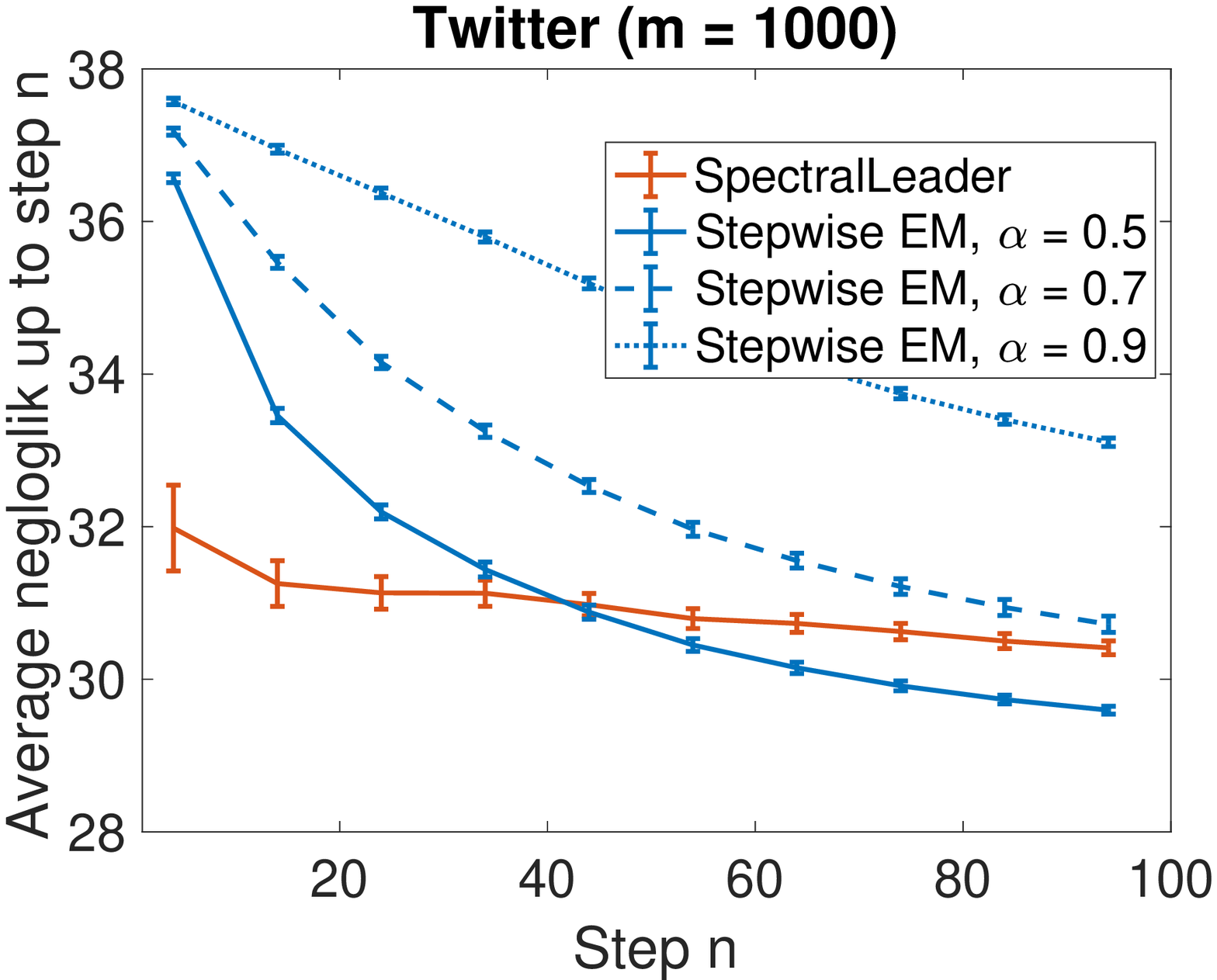}
   \caption{}
    \label{fig:real_c}
  \end{subfigure} 
  \begin{subfigure}[b]{0.405\textwidth}
 \includegraphics[width=1\textwidth]{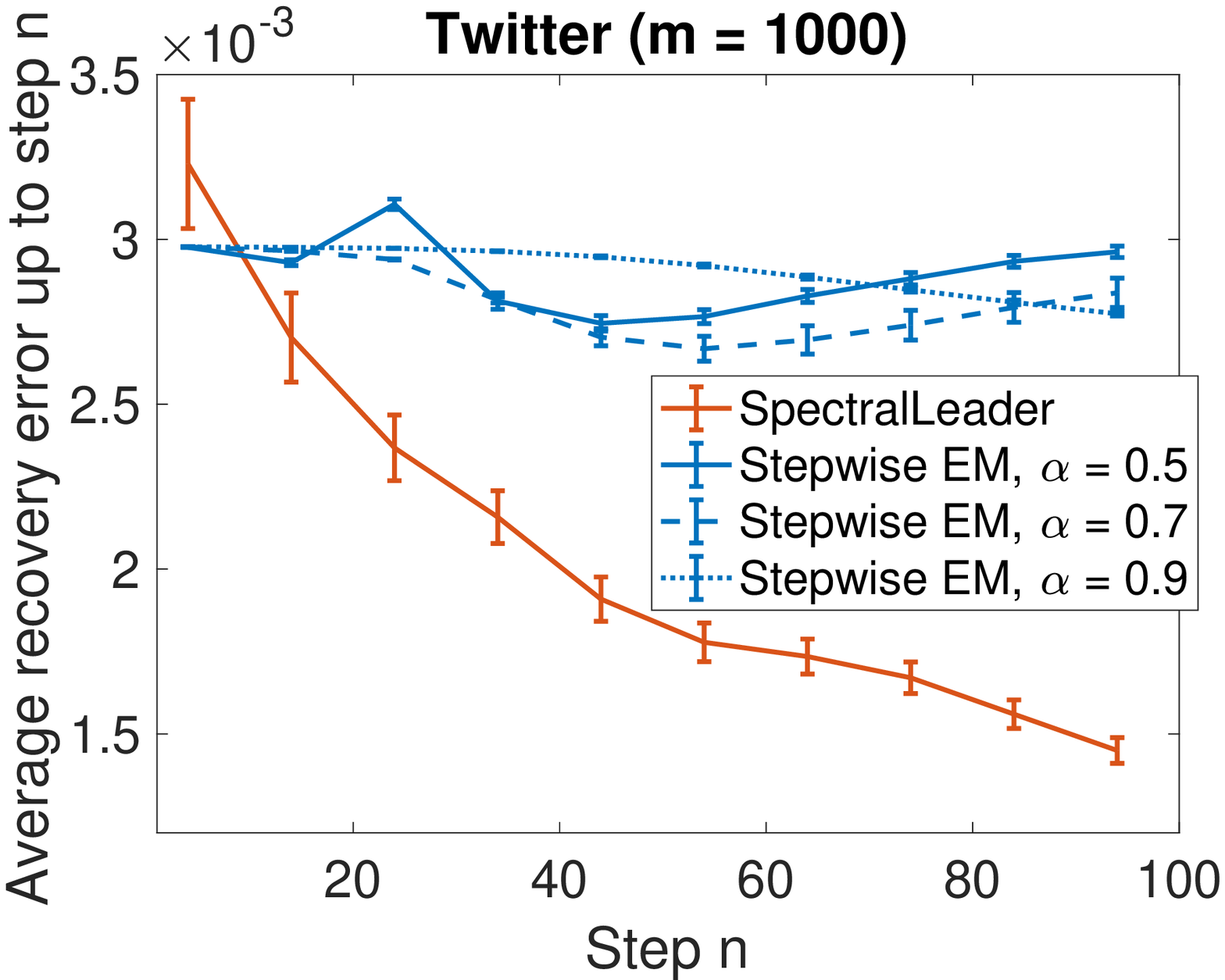}
   \caption{}
    \label{fig:real_g}
  \end{subfigure} \\
  \begin{subfigure}[b]{0.40\textwidth}
 \includegraphics[width=1\textwidth]{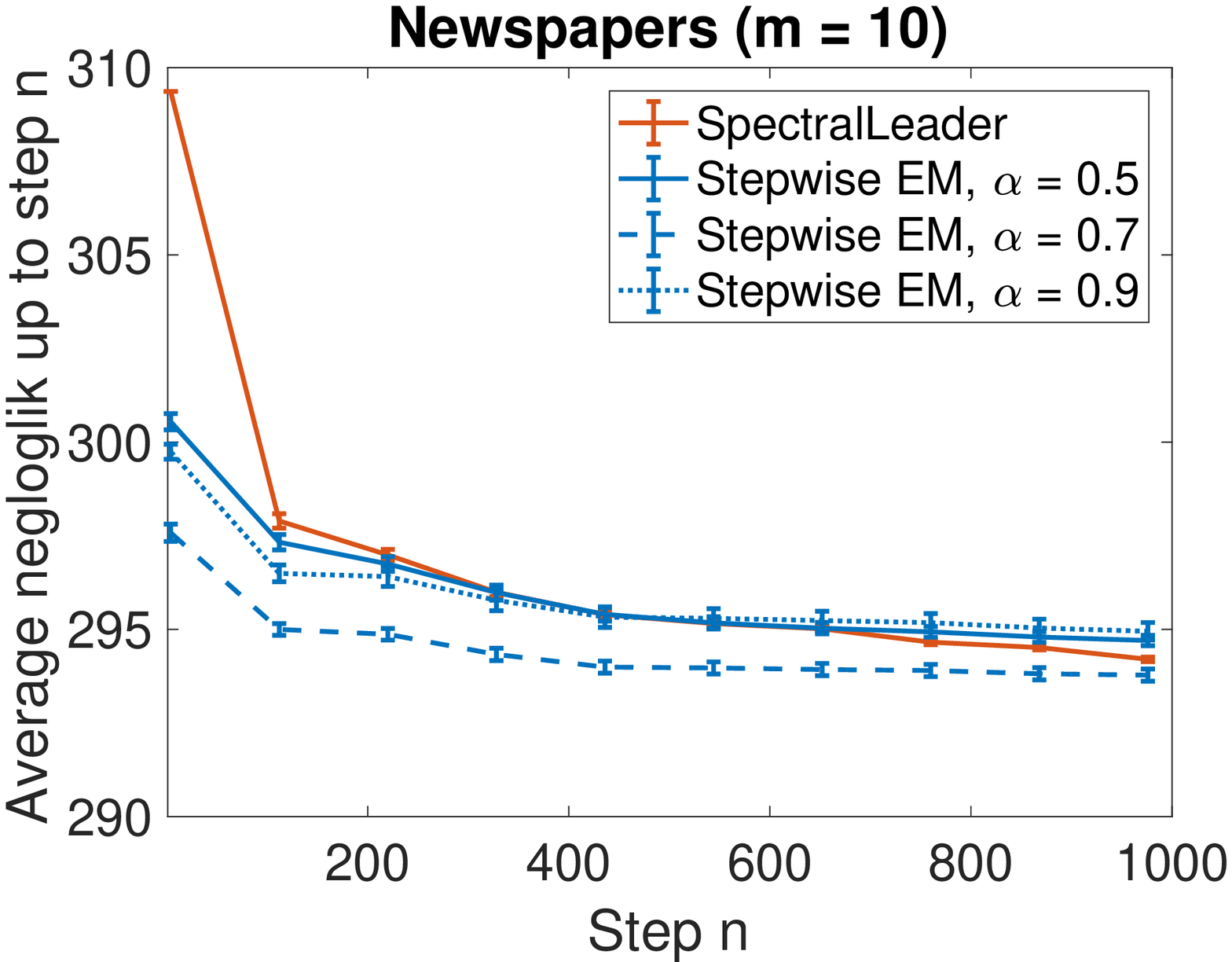}  \caption{}
    \label{fig:real_d}
  \end{subfigure} 
  \begin{subfigure}[b]{0.40\textwidth}
 \includegraphics[width=1\textwidth]{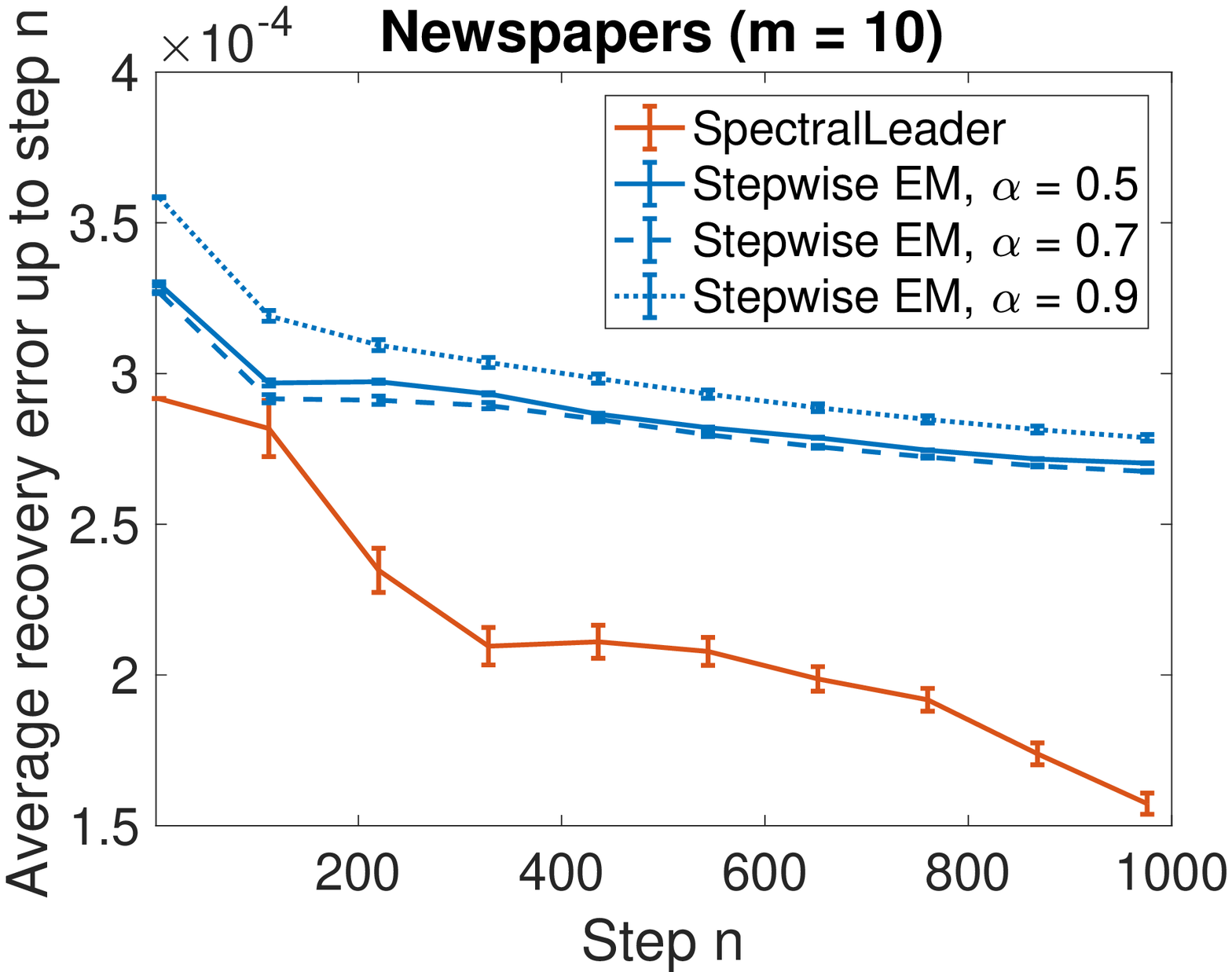}
   \caption{}
    \label{fig:real_h}
  \end{subfigure} \\
  \begin{subfigure}[b]{0.40\textwidth}
 \includegraphics[width=1\textwidth]{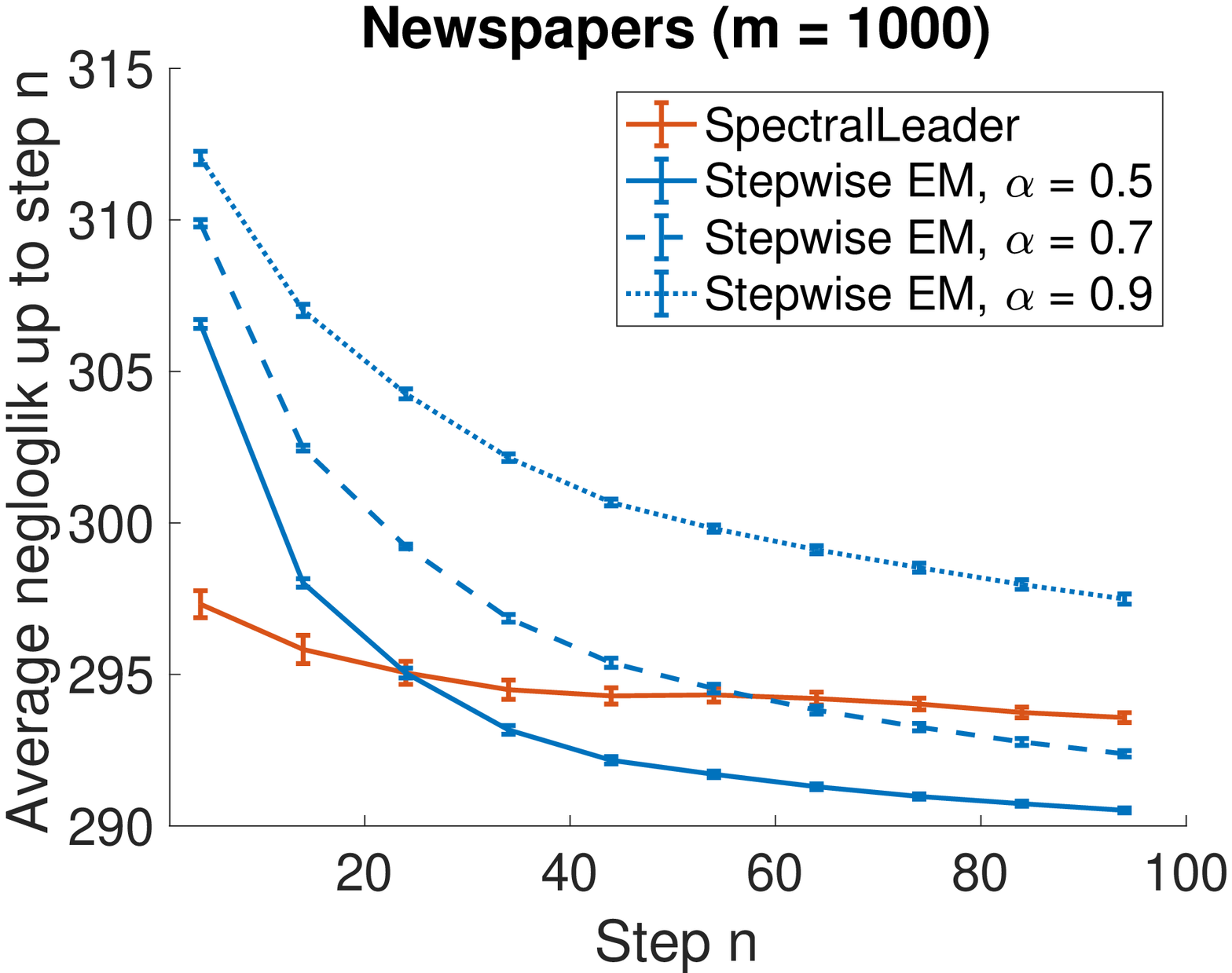}  \caption{}
    \label{fig:real_m}
  \end{subfigure} 
  \begin{subfigure}[b]{0.395\textwidth}
 \includegraphics[width=1\textwidth]{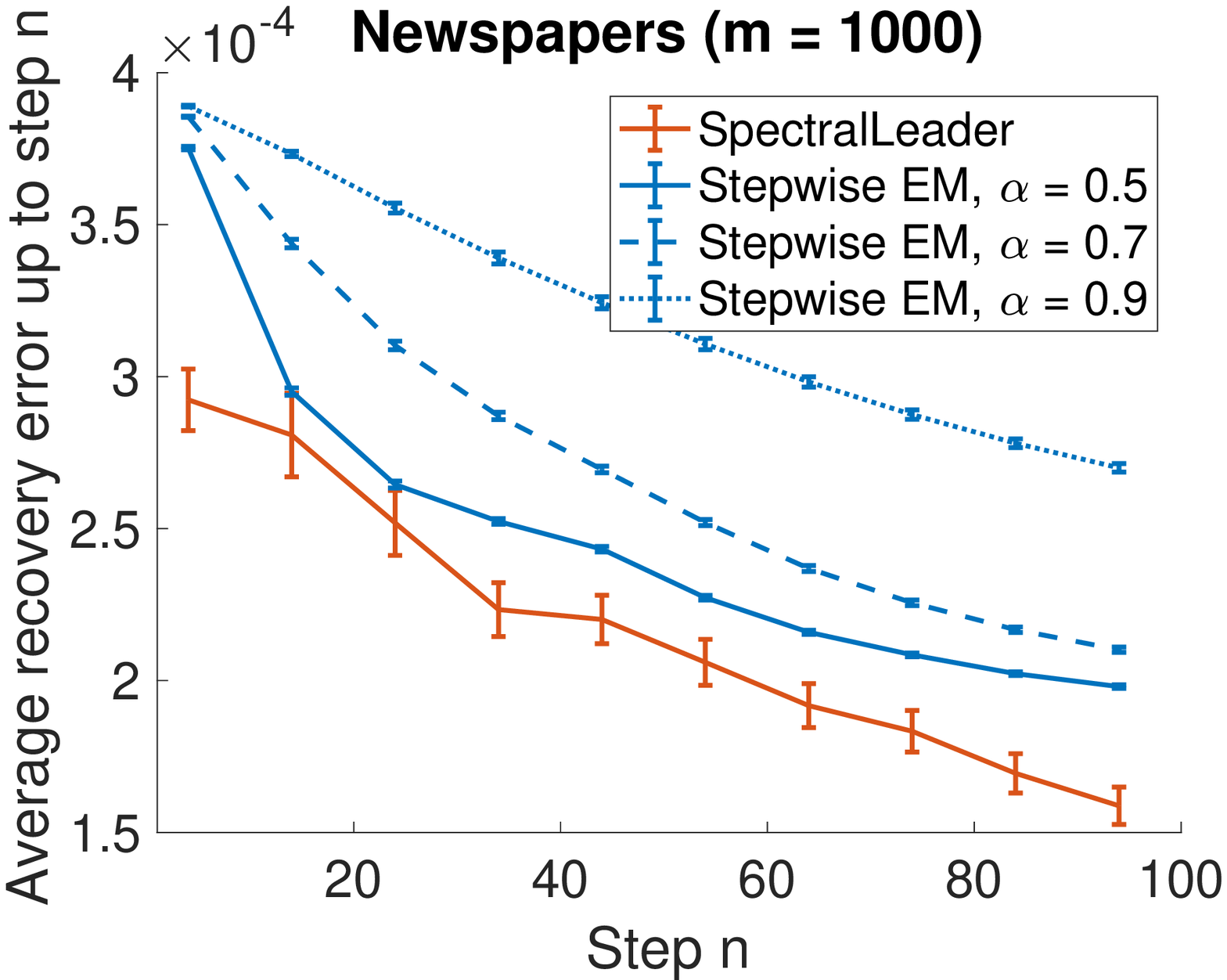}
   \caption{}
    \label{fig:real_n}
  \end{subfigure}
 \caption{The evaluations on two real-world datasets. The first column shows the results under the metric $\mathcal{L}_{n}^{(1)}$, while the second  column shows the results under the metric $\mathcal{L}_{n}^{(2)}$. } 
 \label{fig:realnew1}
 \end{figure*}

\subsection{Synthetic Non-Stochastic Setting}
\label{sec:synnonsto}

Now we evaluate all algorithms on the hard problem in non-stochastic setting. This problem is the same as the hard problem in the stochastic setting, except that topics of the documents are strongly correlated over time. In each batch of $100$ steps, sequentially we have $15$ documents from topic $1$, $35$ documents from topic $2$, and $50$ documents from topic $3$. 

Our results are reported in Figure \ref{fig:non1} and \ref{fig:non2}. We observe two major trends. First, for stepwise EM, the $\alpha$ leading to lowest negative log-likelihood is $0.5$. This result matches well the fact that the smaller the $\alpha$, the faster the old sufficient statistics are forgotten, and the faster stepwise EM adapts to the non-stochastic setting. Second, in terms of adaptation to correlated topics, $\spectralfpl$ is even better than stepwise EM with $\alpha = 0.5$. Note that $\alpha = 0.5$ is the smallest valid value of $\alpha$ for stepwise EM \cite{liang2009online}.

\subsection{Real World Problems}

In this section, we compare $\spectralfpl$ to stepwise EM on real world problems. We evaluate on Newspapers data \footnote{Please see \url{https://www.kaggle.com/snapcrack/all-the-news}.}  collected over multiple years and Twitter data \footnote{Please see \url{https://www.kaggle.com/kinguistics/election-day-tweets}.}  collected during the $2016$ United States elections. They provide sequences of documents with timestamps and the distributions of topics change over time. After preprocessing, we retain $500$ most frequent words in the vocabulary. We set $K = 5$. We evaluate all algorithms on $100$K documents. We compare $\spectralfpl$ to stepwise EM with multiple $\alpha$, and mini-batch sizes $m = 10$ and $m = 1000$. We show the results before the different methods converge: for $m = 10$, we report results before $n = 1000$, and for $m = 1000$ we report results before $n = 100$. 
To handle the large scale streaming data (\emph{e.g.}, $5$M words in Newspapers data), we introduce reservoir sampling, and set the window size of reservoir as $10$,$000$.

Our results are reported in \cref{fig:realnew1}. We observe four major trends. First, under the metric $\mathcal{L}_{n}^{(2)}$, $\spectralfpl$ performs better than stepwise EM. Second, under the metric $\mathcal{L}_{n}^{(1)}$ on both datasets, the optimal $\alpha$ for stepwise EM are different, for $m = 10$ versus $m = 1000$. Third, when $m= 10$, under the metric $\mathcal{L}_{n}^{(1)}$,  $\spectralfpl$ performs competitive with or better than stepwise EM with its best $\alpha$. Fourth, when $m= 1000$, under the metric $\mathcal{L}_{n}^{(1)}$, $\spectralfpl$ is not as good as the stepwise EM with its best $\alpha$. However, directly using $\spectralfpl$ without the effort of tuning any parameters can still provide good performance. These results suggest that, even when the mini-batch size is large, $\spectralfpl$ is still very useful under the log-likelihood metric: in practice we can quickly achieve reasonable results by $\spectralfpl$ without any parameter tuning.

%% file: Conclusion.tex

\section{Conclusions}

We propose $\spectralfpl$, a novel online learning algorithm for latent variable models. In an instance of a bag-of-words model, we define a novel per-step loss function, prove that $\spectralfpl$ converges to a global optimum, and derive a sublinear cumulative regret bound for $\spectralfpl$. Our experimental results show that $\spectralfpl$ performs similarly to or better than the optimally-tuned online EM. In our future work, we want to extend our method to more complicated latent-variable models, such as HMMs and LDA \cite{anandkumar2014tensor}.